\documentclass{article}
\usepackage[utf8]{inputenc}
\usepackage{a4wide}
\usepackage{xcolor}
\usepackage{amsmath}
\usepackage{lscape}
\usepackage[font=small,labelfont=bf]{caption}
\usepackage{graphicx}
\usepackage{subcaption}
\usepackage{comment}
\usepackage{tabularx}
\usepackage{multirow}
\usepackage{algorithmic}
\usepackage[ruled,vlined]{algorithm2e}
\usepackage{amsmath,amssymb,amsfonts,amsthm} 
{
\newtheorem{theorem}{Theorem}

\theoremstyle{remark}
\newtheorem {remark}{Remark}

}
\usepackage[utf8]{inputenc}
\usepackage{amssymb}
\usepackage{multicol}
\usepackage{geometry}
\usepackage{hyperref}
\usepackage{authblk}
\usepackage{ulem}
\usepackage{moreverb}
\usepackage{amsmath,bm}
\usepackage{graphicx}
\usepackage{siunitx}
\usepackage{subcaption}

\usepackage{booktabs}
\usepackage{pifont}  
\newcommand{\cmark}{\ding{51}} 
\newcommand{\xmark}{\ding{55}} 
\usepackage{threeparttable}

\renewenvironment{abstract}
  {\small\noindent\textbf{Abstract.}\normalsize\ignorespaces}
  {\par\noindent\ignorespacesafterend}

\makeatletter
\renewcommand{\maketitle}{%
    \noindent\fcolorbox{red}{white}{
        \parbox{\textwidth}{%
            \color{red}\copyright2025 IEEE. Accepted for \textit{IEEE Transactions on Cognitive and Developmental Systems}. Personal use of this material is permitted. Permission from IEEE must be obtained for all other uses, in any current or future media, including reprinting/republishing this material for advertising or promotional purposes, creating new collective works, for resale or redistribution to servers or lists, or reuse of any copyrighted component of this work in other works.
        }%
    }
    \vspace{1em} 
    \begin{center}%
        {\LARGE \@title \par}
        \vspace{0.5em}
        {\large \@author \par}
        \vspace{0.5em}
        {\large \@date}
    \end{center}%
    \vspace{1em} 
}
\makeatother

\begin{document}

\title{A Minimal Model for Emergent Collective Behaviors in Autonomous Robotic Multi-Agent Systems}

\author{Hossein B. Jond}
\affil{Department of Cybernetics, Czech Technical University in Prague, 12000 Prague, Czechia}

\date{}

\maketitle

\begin{abstract} 
Collective behaviors such as swarming and flocking emerge from simple, decentralized interactions in biological systems. Existing models, such as Vicsek and Cucker-Smale, lack collision avoidance, whereas the Olfati-Saber model imposes rigid formations, limiting their applicability in swarm robotics. To address these limitations, this paper proposes a minimal yet expressive model that governs agent dynamics using relative positions, velocities, and local density, modulated by two tunable parameters: the spatial offset and kinetic offset. The model achieves spatially flexible, collision-free behaviors that reflect naturalistic group dynamics. Furthermore, we extend the framework to cognitive autonomous systems, enabling energy-aware phase transitions between swarming and flocking through adaptive control parameter tuning. This cognitively inspired approach offers a robust foundation for real-world applications in multi-robot systems, particularly autonomous aerial swarms.
 
\end{abstract}

\section{Introduction}
Collective behavior in biological systems, from bacteria to vertebrates, exhibits striking similarities in emergent global patterns and information transfer, as observed in bird flocks~\cite{CAVAGNA20181}, insect swarms~\cite{Wang2786790}, and fish schools~\cite{Ito}. These phenomena arise from simple, decentralized interactions among entities within a limited range.
Reynolds~\cite{Reynolds} introduced three local interaction rules: separation, alignment, and cohesion, in his \textit{boids} model, which simulates realistic collective motion like flocking and predator evasion.

Subsequent efforts have mathematically modeled collective behaviors. The notable models include the Vicsek model~\cite{Vicsek}, zone-based interaction model~\cite{COUZIN20021}, Olfati-Saber model~\cite{Olfati-Saber}, and Cucker-Smale model~\cite{Cucker-Smale}. Many variations of the Vicsek model for particle systems have been proposed to mimic the properties of natural creatures~\cite{PhysRevLett.114.168001,PhysRevLett.127.258001,PhysRevE.100.062415,ZHOU2024133983}. The Vicsek model describes self-propelled particles moving at a constant velocity, aligning their direction with the average orientation of nearby particles, and transitioning from disorder to order by tuning angular noise and density terms. The zone-based interaction model organizes local interactions into three zones: repulsion, orientation, and attraction, enabling the replication of collective behaviors through zone adjustments. Some variants of this model have been proposed to control swarm motion~\cite{10375747,Jond2025}. Olfati-Saber introduced three types of agents: the \(\alpha\)-agents (cooperative agents), \(\beta\)-agents (obstacles), and \(\gamma\)-agents (leaders). Pairwise potential-based control laws are designed for \(\alpha\)-agent interactions to form \(\alpha\)-lattices, typically hexagonal, and to manage interactions with obstacles and leaders. The \(\alpha\)-lattice flock formation has gained significant attention in control applications~\cite{9998071,10644765,7112591}. The Cucker-Smale model focused on the alignment of velocity influenced by relative positions, where the alignment strength is dependent on distance. This model has been the subject of extensive study, with numerous extensions and variants proposed~\cite{S0218202516500287,CHOI201849,Choi2017}.

Nevertheless, the interaction rules in these models and their variations often fail to replicate the properties that make Reynolds' \textit{boids} excel in producing complex collective behaviors. For instance, the Vicsek model, zone-based interaction model, and Cucker-Smale model lack critical features such as flock centering, inter-agent collision avoidance, or obstacle avoidance. Meanwhile, the Olfati-Saber model tends to produce rigid and overly regular patterns, missing the flexibility inherent in the \textit{boids} model. 

The ability to avoid collisions and maintain personal space is fundamental in both biological and robotic swarms. Classical models do not inherently address this aspect of collective motion without incorporating complex repulsive potentials, thereby compromising their minimalism. This highlights a research gap for innovative models that effectively incorporate collision avoidance while preserving simplicity and computational efficiency.

The advantages of studying collective behavior can be considered to be twofold. This helps explain many natural phenomena and physical information transfer mechanisms. Additionally, it inspired the development of three types of swarm robots~\cite{10.1093/nsr/nwad040}, including aerial swarms, such as swarming drones~\cite{scirobotics.abm5954,8594266} and flocking drones~\cite{noauthor_optimized_nodate,10220161}, ground swarms, such as mini-wheeled robots~\cite{STOLFI2024107501}, marine swarms, such as robot fish~\cite{scirobotics.abd8668}, and combined swarms, such as aerial-ground swarms~\cite{scirobotics.adl5161}.

This study presents a minimal mathematical model designed to replicate natural swarming and flocking behaviors in robotic multi-agent systems. The velocity dynamics of each agent are shaped by local interactions that incorporate relative position, relative velocity, and local density. These decentralized interactions give rise to diverse emergent behaviors, ranging from highly coordinated flocking to disordered swarming, depending on the values of key control parameters. The velocity dynamics are modulated by two tunable parameters: the kinetic offset and the spatial offset. The kinetic offset governs the degree of alignment between neighboring agents’ velocities, enabling a continuum of behavior from disordered swarming to ordered flocking. The spatial offset regulates inter-agent spacing, implicitly supporting collision avoidance by maintaining personal space among agents. The model’s distributed structure ensures scalability and robustness, making it well-suited for swarm robotics applications. Simulation results demonstrate the model's capacity to reproduce a wide range of collective motion patterns, including swarming, vortexing, and flocking. Furthermore, a comprehensive analysis is conducted to examine the influence of control parameter variations on the emergence of collective behaviors across different population sizes. The results demonstrate that the model exhibits a continuum phase transition, particularly capturing the gradual shift between flocking and swarming. This paper’s primary contribution is a minimal model that generates spatially flexible, collision-free swarming and flocking behaviors emerging from simple local interactions. The model overcomes the limited realism of the Vicsek and Cucker-Smale models, which lack collision avoidance, and the rigid flocking of the Olfati-Saber model, thereby enhancing its applicability to swarm robotics.

Furthermore, the model is extended to incorporate energy-aware adaptation, allowing agents to dynamically adjust their spatial and kinetic offsets based on their individual energy levels. Agents with ample energy reserves maintain higher kinetic offsets, favoring swarming behavior characterized by sustained motion and decentralized exploration. In contrast, agents with diminished energy reduce both offsets, transitioning toward flocking behavior that promotes energy efficiency through compact coordinated movement and reduced relative velocities. This adaptive mechanism captures biologically inspired trade-offs between cooperative group dynamics and individual energy constraints. By embedding such resource-aware decision-making, the model advances the realism and applicability of cognitive multi-robot systems in scenarios where long-term autonomy and energy efficiency are critical—such as UAV-based shepherding~\cite{9173524} and dynamic target encirclement~\cite{9050630}.

The remainder of this paper is organized as follows. Section~\ref{sec:col-beh} provides an overview of emergent collective behaviors and models. Sections~\ref{sec:model},~\ref{sec:model-cluttered}, and~\ref{sec:model-ext} describe the model and its extensions incorporating target-directed and obstacle avoidance potentials to enable navigation toward global objectives in cluttered environments, and resource-aware cognitive extension, respectively. Section~\ref{sec:global} examines the asymptotic stability of the proposed framework. Section~\ref{sec:sim} provides simulation results along with model analyses. Finally, Section~\ref{sec:con} offers concluding remarks and outlines potential directions for future work.

\section{Emergent Collective Behaviors and Models}\label{sec:col-beh}

Flocking, vortexing, and swarming are emergent behaviors in biological groups, driven by simple local interaction rules. Table~\ref{tab:swarm_flock} summarizes their characteristics.

\begin{table}[t]
\centering
\caption{Characteristics of flocking, vortexing, and swarming.}
\begin{tabular}{lccc}
\toprule
Aspect & Flocking & Vortexing & Swarming \\
\midrule
Coordination & High & Moderate & Low \\
Alignment & Strong & Minimal & Minimal \\
Emergent Behavior & Cohesion & Spiral & Aggregation \\
Organization & Ordered & Structured disorder & Disordered \\
Biological Goals & Migration & Protection & Foraging \\
Robotic Goals & Navigation & Rotational movement & Exploration \\
\bottomrule
\end{tabular}
\label{tab:swarm_flock}
\end{table}

These behaviors arise from rules based on position, velocity, and local density, balancing attraction, repulsion, and alignment~\cite{KOLOKOLNIKOV20131}. Flocking involves strong alignment and aggregation for coordinated movement, as in migrating birds~\cite{BAJEC2009777}. Vortexing features agents circling a center, balancing local structure and global disorder for protection~\cite{Delcourt}. Swarming emphasizes aggregation without alignment, enabling flexible responses for foraging or evasion~\cite{Chakraborty2020-tm}.

Reynolds’ \textit{boids} model~\cite{Reynolds} uses collision avoidance, velocity matching, and flock centering. Olfati-Saber’s model~\cite{Olfati-Saber} formalizes these but yields rigid lattices, unlike \textit{boids}’ flexibility. Vicsek~\cite{Vicsek} and Cucker-Smale~\cite{Cucker-Smale} models enable flocking but lack collision avoidance. Our model achieves flexible, collision-free swarming and flocking, overcoming Vicsek and Cucker-Smale’s limitations and Olfati-Saber’s rigidity, enhancing swarm robotics applicability. Table~\ref{tab:s} compares these models.

\begin{table}[b]
\centering
\begin{threeparttable}
\caption{Emergent behaviors of models.}
\label{tab:s}
\begin{tabular}{lccc}
\toprule
Model & Flocking & Swarming & Collision Avoidance \\
\midrule
Reynolds\tnote{1} & \cmark & \xmark & \cmark \\
Vicsek & \cmark & \cmark & \xmark \\
Cucker-Smale & \cmark & \xmark & \xmark \\
Olfati-Saber\tnote{2} & \cmark & \xmark & \cmark \\
\textbf{Proposed}\tnote{3} & \cmark & \cmark & \cmark \\
\bottomrule
\end{tabular}
\begin{tablenotes}
\item[1]Heuristic, rule-based model.
\item[2]Uses repulsive potentials for strong separation.
\item[3]Employs offset-vector-based repulsion for soft, effective separation.
\end{tablenotes}
\end{threeparttable}
\end{table}

Consider $n$ interacting agents ($n \geq 2$) navigating an $m$-dimensional space, where $m = 2$ or $3$. For each agent indexed by $i = 1, \ldots, n$, the position and velocity vectors are denoted by $\mathbf{p}_i \in \mathbb{R}^m$ and $\mathbf{v}_i \in \mathbb{R}^m$, respectively. The Cucker-Smale model describes the flocking behavior through velocity alignment dynamics. Specifically, each agent adjusts its velocity by incorporating a weighted average of the velocity differences between itself and other agents. The Cucker-Smale model is governed by
\begin{align}\label{eq:CS-model}
    \frac{d\mathbf{p}_i}{dt}&=\mathbf{v}_i, \notag\\
    \frac{d\mathbf{v}_i}{dt}&=\sum_{j=1}^n\omega(\lVert \mathbf{p}_j-\mathbf{p}_i\rVert)(\mathbf{v}_j-\mathbf{v}_i),
\end{align}
where the weight function $\omega$ is defined as
\[
\omega(\lVert \mathbf{p}_j-\mathbf{p}_i\rVert)=\frac{K}{(\sigma^2+\lVert \mathbf{p}_j-\mathbf{p}_i\rVert)^\gamma},
\]
with constants $K, \sigma > 0$, $\gamma \geq 0$, and $\lVert \cdot \rVert$ represents the Euclidean norm. The weight function depends on the distance between agents, influencing the strength of their velocity alignment.

Extensions of the Cucker-Smale model address its lack of collision avoidance. For instance, interparticle bonding forces were introduced to ensure collision-free flocking~\cite{Park2010}, while sharp conditions for avoiding collisions in singular Cucker-Smale interactions were established for general initial configurations~\cite{Carrillo2017}. Notably, collision avoidance without additional forces was proven for some singular short-range-communicated Cucker-Smale models~\cite{Yin2020}. 

\begin{remark}
The proposed model may be viewed as a generalization of the Cucker-Smale flocking model, incorporating several notable distinctions. First, while the Cucker-Smale model primarily emphasizes alignment-based flocking behavior, our model introduces a spatial offset tuning parameter that regulates inter-agent distances, thereby implicitly promoting collision avoidance. Second, in contrast to the Cucker-Smale model's distance-based weight function for velocity alignment, our model employs a relative speed-based influence function, governed by the kinetic offset parameter. A low kinetic offset yields flocking via velocity consensus, while a high offset enables swarming through sustained velocity differences.
\end{remark}

\begin{remark}
The kinetic offset functions similarly to noise in the Vicsek model: smaller values promote alignment, like low noise, while larger values allow disordered motion, akin to high noise levels. 
\end{remark}

\section{The Proposed Collective Behavior Model}\label{sec:model}

Our model comprises $n$ interacting agents ($n \geq 2$) navigating an $m$-dimensional space. The neighborhood of agent $i$ is defined as $\mathcal{N}_i = \{j \mid j \neq i \text{ and } \lVert \mathbf{p}_j - \mathbf{p}_i \rVert \leq r_i\}$, where $r_i > 0$ denotes the interaction radius of agent $i$.

The agent dynamics are governed by the following interaction model
\begin{align}\label{eq:model}
    \frac{d\mathbf{p}_i}{dt}&=\mathbf{v}_i, \notag\\
    \frac{d\mathbf{v}_i}{dt}&=\underbrace{\sum_{j\in\mathcal{N}_i}\psi(\lVert \mathbf{p}_j-\mathbf{p}_i\rVert)(\mathbf{p}_j-\mathbf{p}_i)}_{\text{aggregation potential}}\notag\\&+\underbrace{\sum_{j\in\mathcal{N}_i}\phi(\lVert \mathbf{v}_j-\mathbf{v}_i\rVert)(\mathbf{v}_j-\mathbf{v}_i)}_{\text{alignment potential}},
\end{align}
where the interaction weights are defined as
\begin{align*}
    &\psi(\lVert \mathbf{p}_j-\mathbf{p}_i\rVert)=1-(\frac{\delta_i |\mathcal{N}_i|}{\lVert \mathbf{p}_j-\mathbf{p}_i\rVert})^{\alpha_i},\\
    &\phi(\lVert \mathbf{v}_j-\mathbf{v}_i\rVert)=1- (\frac{\eta_i}{|\mathcal{N}_i|\lVert \mathbf{v}_j-\mathbf{v}_i\rVert})^{\beta_i},
\end{align*}
where \(\delta_i,\eta_i \geq 0\) are the spatial and kinetic offsets, respectively, \(|\mathcal{N}_i|\) denotes the cardinality of the set \(\mathcal{N}_i\), and \(\alpha_i,\beta_i>0\) control the steepness of the response.

The model balances aggregation (cohesion-separation) and alignment potentials. The spatial offset $\delta_i$ regulates inter-agent spacing, promoting denser formations when small and looser spacing when large, implicitly ensuring collision avoidance. The kinetic offset $\eta_i$ controls velocity alignment: smaller values enhance consensus for flocking, while larger values allow velocity diversity for swarming or vortexing. Offsets $\delta_i$ and $\eta_i$ tune transitions between flocking and swarming, with $\alpha_i$ and $\beta_i$ adjusting interaction sensitivity. Weight function profiles for $\psi$ and $\phi$ are shown in Figs.~\ref{fig:weights}(\subref{fig:psi}) and~\ref{fig:weights}(\subref{fig:phi}). Table~\ref{tab:parameter_effects} summarizes their effects.

\begin{figure}[t]
   \centering
   \begin{subfigure}{0.475\linewidth}
       \centering
       \includegraphics[width=\linewidth]{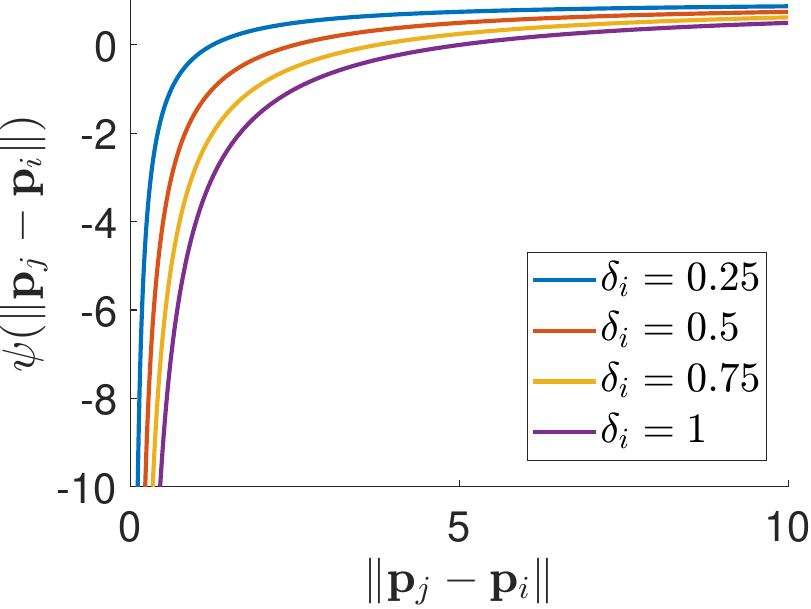}
       \subcaption[]{} \label{fig:psi}
   \end{subfigure}
   \begin{subfigure}{0.475\linewidth}
       \centering
       \includegraphics[width=\linewidth]{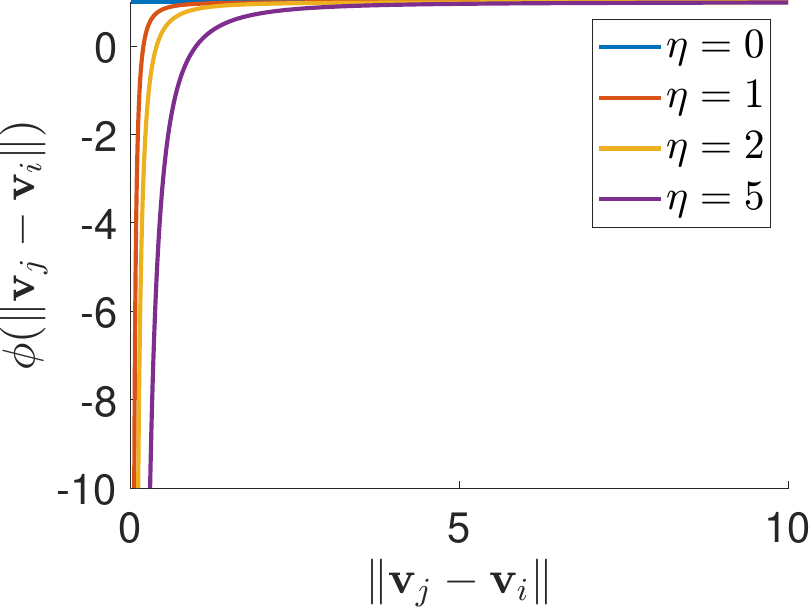}
       \subcaption[]{} \label{fig:phi}
   \end{subfigure}
   \caption{(a) Aggregation weight \(\psi\) for \(\delta_i = 0.25, 0.5, 0.75, 1\), \(\alpha = 1\), \(|\mathcal{N}_i| = 5\). Positive \(\psi\) induces cohesion, negative \(\psi\) induces separation, zero at \(\lVert \mathbf{p}_j - \mathbf{p}_i \rVert = \delta_i |\mathcal{N}_i|\). (b) Alignment weight \(\phi\) for \(\eta_i = 0, 1, 2, 5\), \(\beta = 2\), \(|\mathcal{N}_i| = 5\). Positive \(\phi\) promotes alignment, negative \(\phi\) reduces alignment, zero at \(\lVert \mathbf{v}_j - \mathbf{v}_i \rVert = \frac{\eta_i}{|\mathcal{N}_i|}\).}
   \label{fig:weights}
\end{figure}

\begin{table}[t]
\centering
\caption{Effects of model parameters on collective dynamics.}
\begin{tabular}{cp{6.5cm}}
\toprule
Parameter & Effect on Collective Behavior \\
\midrule
$\delta \downarrow$ & Denser aggregation (flocking) \\
$\delta \uparrow$ & Looser spacing (swarming) \\
$\eta \downarrow$ & Stronger velocity alignment (flocking) \\
$\eta \uparrow$ & Velocity misalignment (swarming) \\
$\alpha, \beta \uparrow$ & Steeper interaction responses \\
$\alpha, \beta \downarrow$ & Smoother interaction responses \\
\bottomrule
\end{tabular}
\label{tab:parameter_effects}
\end{table}

In the context of robotic swarms, it is imperative that the velocity of each agent is restricted within its designated physical limits to prevent saturation, a condition in which a control input exceeds the operational limits of the system. To address this within the proposed model, a smooth saturation mechanism based on the hyperbolic tangent function is employed. This function effectively constrains the velocity magnitude without abrupt clipping, thereby preserving continuity in the control signal and avoiding destabilizing discontinuities. Accordingly, the velocity is regulated as follows,
\[
   \mathbf{v}_i = v_i^{\max} \tanh{\left(\frac{\lVert \mathbf{v}_i \rVert}{v_i^{\max}}\right)} \frac{\mathbf{v}_i}{\lVert \mathbf{v}_i \rVert}, 
\]
where \( v_i^{\max} \) is the maximum speed for the agent \(i\). If the computed velocity magnitude exceeds the agent's limits, it is smoothly capped at \( v_i^{\max} \). The saturation value \(s_i\) for the magnitude of the velocity rate is determined based on the time to reach the maximum speed, 
\[
s_i = \frac{v_i^{\max}}{t_i^{v_{\max}}},
\]
where \( t_i^{v_{\max}} \) is the time required to reach \( v_i^{\max} \).  

Consider the expanded formulation of the interaction model
\begin{align*}
    \frac{d\mathbf{p}_i}{dt}&=\mathbf{v}_i, \notag\\
    \frac{d\mathbf{v}_i}{dt}&=\sum_{j\in\mathcal{N}_i}(\mathbf{p}_j-\mathbf{p}_i-\frac{(\delta_i |\mathcal{N}_i|)^\alpha}{\lVert \mathbf{p}_j-\mathbf{p}_i\rVert^{\alpha-1}}\frac{\mathbf{p}_j-\mathbf{p}_i}{\lVert \mathbf{p}_j-\mathbf{p}_i\rVert})\notag\\&+\sum_{j\in\mathcal{N}_i}(\mathbf{v}_j-\mathbf{v}_i- \frac{\eta_i^\beta}{|\mathcal{N}_i|^\beta\lVert \mathbf{v}_j-\mathbf{v}_i\rVert^{\beta-1}}\frac{\mathbf{v}_j-\mathbf{v}_i}{\lVert \mathbf{v}_j-\mathbf{v}_i\rVert}),
\end{align*}
which explicitly decomposes the interactions into directional and magnitude-modulated components. In this formulation, the unit vectors $\frac{\mathbf{p}_j - \mathbf{p}_i}{\lVert \mathbf{p}_j - \mathbf{p}_i \rVert}$ and $\frac{\mathbf{v}_j - \mathbf{v}_i}{\lVert \mathbf{v}_j - \mathbf{v}_i \rVert}$ encode the direction of inter-agent displacement and relative velocity, respectively. These unit directions are scaled by offset-dependent terms $\frac{(\delta_i |\mathcal{N}_i|)^\alpha}{\lVert \mathbf{p}_j - \mathbf{p}_i \rVert^{\alpha - 1}}$ and $\frac{\eta_i^\beta}{|\mathcal{N}_i|^\beta \lVert \mathbf{v}_j - \mathbf{v}_i \rVert^{\beta - 1}}$, which regulate the strength of aggregation and alignment forces based on spatial and kinetic offsets.

The displacement offset vector drives repulsion or attraction, enforcing equilibrium spacing at $\delta_i |\mathcal{N}_i|$, while the velocity offset vector governs alignment, enabling flocking or swarming based on $\eta_i$. Implicit collision avoidance is embedded via the repulsion mechanism, where the interaction weight $\psi(\lVert \mathbf{p}_j - \mathbf{p}_i \rVert)$ becomes strongly negative as $\lVert \mathbf{p}_j - \mathbf{p}_i \rVert \to 0$, pushing agents apart to maintain a minimum separation distance without additional constraints.

The model is scalable, as each agent interacts only with neighbors within $r_i$, keeping computational complexity independent of $n$. Agents adjust behavior based on local neighborhood $\mathcal{N}_i$ and relative positions/velocities. Including $|\mathcal{N}_i|$ in spatial ($\delta_i$) and kinetic ($\eta_i$) offsets adaptively scales neighbor influence by density. The separation equilibrium distance $|\mathcal{N}_i|\delta_i$ increases in crowded settings for more space and decreases in sparse ones for cohesion. This ensures stable, flexible dynamics in large-scale, varying-density environments.

\section{Target-Directed Collective Behavior in Cluttered Environments}~\label{sec:model-cluttered}
To account for migration toward a preferred target location and obstacle avoidance in cluttered environments, we extend the collective behavior model introduced in Section~\ref{sec:model}. These extensions are inspired by natural collectives, where agents orient their motion toward a common goal while avoiding collisions with environmental obstacles. For robotic multi-agent systems, the interaction rules incorporate: ($i$) migration toward a target location, which steers agents toward a shared objective; ($ii$) obstacle avoidance, which prevents collisions with environmental obstacles; and ($iii$) the existing aggregation and alignment interactions, which regulate inter-agent cohesion, separation, and velocity alignment.

The target location, denoted by \(\mathbf{p}_t \in \mathbb{R}^m\), is known to all agents and establishes a common direction for the collective. The environment contains a set of \(O\) obstacles, denoted by \(\mathcal{O} = \{1, \dots, O\}\), with each obstacle \(o \in \mathcal{O}\) located at position \(\mathbf{p}_o \in \mathbb{R}^m\). The set of obstacles detected by agent \(i\) is defined as \(\mathcal{O}_i = \{o \in \mathcal{O} \mid \lVert \mathbf{p}_o - \mathbf{p}_i \rVert \leq c_i\}\), where \(c_i > 0\) is the obstacle detection/interaction radius of agent \(i\).

The extended dynamics of agent \(i\) in a cluttered environment with a target are governed by the following interaction model, which augments the original dynamics in \eqref{eq:model}:
\begin{align}\label{eq:extended_model}
    \frac{d\mathbf{p}_i}{dt} &= \mathbf{v}_i, \notag\\
    \frac{d\mathbf{v}_i}{dt} &= \text{aggregation potential}+ \text{alignment potential} \notag\\
    &\quad + \underbrace{\kappa_i (\mathbf{p}_t - \mathbf{p}_i)}_{\text{target potential}} + \underbrace{\sum_{o \in \mathcal{O}_i} \rho(\lVert \mathbf{p}_o - \mathbf{p}_i \rVert)(\mathbf{p}_o - \mathbf{p}_i)}_{\text{obstacle avoidance potential}},
\end{align}
where \(\kappa_i > 0\) is a constant weight controlling the strength of attraction to the target, and the obstacle avoidance interaction weight is defined as
\[
\rho(\lVert \mathbf{p}_o - \mathbf{p}_i \rVert) = \min(0, \, 1 - ( \frac{c_i}{\lVert \mathbf{p}_o - \mathbf{p}_i \rVert} )^{\sigma_i} ),
\]
with \(\sigma_i > 0\) controlling the steepness of the repulsion response.

The obstacle avoidance interaction weight \(\rho\) is designed to induce repulsion from obstacles, analogous to the repulsion component of the aggregation interaction weight \(\psi\). The function \(\rho\) decreases sharply as the distance \(\lVert \mathbf{p}_o - \mathbf{p}_i \rVert\) approaches zero, generating a strong repulsive force to steer agent \(i\) away from obstacle \(o\). This behavior ensures effective collision avoidance. The profile of \(\rho\) is illustrated in Fig.~\ref{fig:rho}.

\begin{figure}[t]
    \centering
    \includegraphics[width=0.5\linewidth]{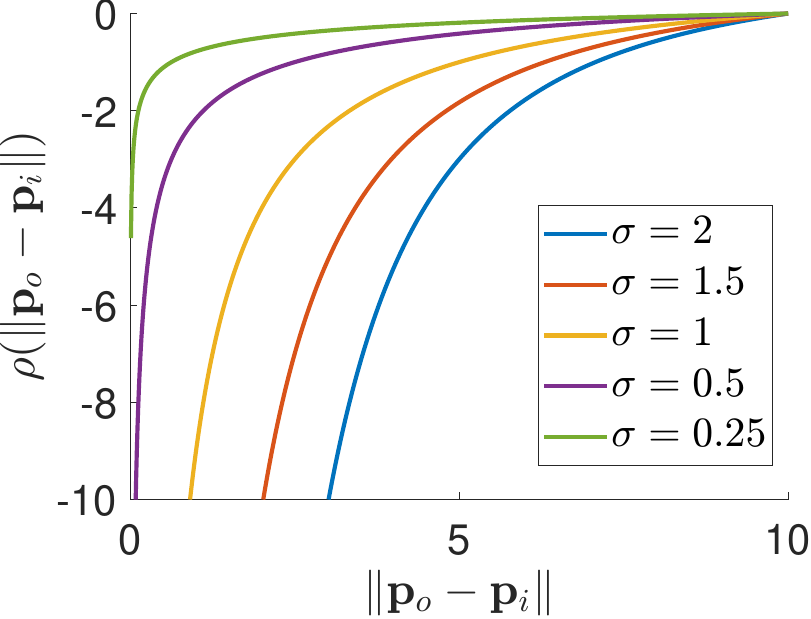}
\caption{Obstacle avoidance weight \(\rho\) for \(\sigma = 2, 1.5, 1, 0.5, 0.25\) at detection/interaction radius \(c_i = 10\).}
\label{fig:rho}
\end{figure}

\section{The Cognitive Adaptation Mechanism}~\label{sec:model-ext}

Cognitive models enable autonomous robotic collectives to adaptively respond to dynamic environments through context-sensitive actions aligned with goals like survival or resource optimization. We frame the transition between flocking and swarming as an emergent cognitive process wherein agents continuously re-evaluate and prioritize competing behavioral imperatives, namely, the maintenance of group cohesion versus the pursuit of environmental exploration. This formulation draws inspiration from natural systems, where animal groups exhibit collective behavioral shifts between flocking during migration for energy efficiency and dispersal during foraging for resource acquisition.

In a robotic multi-agent network, agents endowed with sufficient energy reserves exhibit swarming behavior to perform exploratory or coverage tasks in the environment. Conversely, when energy levels deplete, agents shift their priorities toward resource acquisition, such as navigating to a refueling station, where cohesive movement becomes advantageous and flocking behavior is favored for coordinated navigation. Thus, within a cognitively inspired framework, each agent dynamically adapts its behavior in response to both internal states (e.g., energy availability) and external task demands, resulting in emergent collective adaptation that aligns with group-level objectives.

Assume that the energy expenditure of each agent is expressed as
\[
\frac{dE_i}{dt} = -c_{i1}\lVert\frac{d\mathbf{v}_i}{dt}\rVert^2-c_{i2},\quad E_i(0)=E_{i0},
\]
where \(c_{i1},c_{i2}>0\), \(E_{i0}\) denotes the initial energy level at \(t = 0\), and the squared norm term directly determines the rate of energy consumption.  

For a cognitive energy-driven transition between flocking and swarming behaviors, the individual actions \(\delta_i\) and \(\eta_i\) are defined as functions of the agent's energy level as follows,
\begin{align} 
    &\delta_i =  \delta_{\text{min}} +\frac{\delta_{\text{max}}-\delta_{\text{min}}}{1 + e^{-k_\delta (E_i - E_{\text{th}, i})}}, \label{eq:adaptive_d} \\  
    &\eta_i =  \eta_{\text{min}} + \frac{\eta_{\text{max}}-\eta_{\text{min}}}{1 + e^{-k_\eta (E_i - E_{\text{th}, i})}}, \label{eq:adaptive_e}
\end{align}  
where \(\delta_{\text{min}}\), \(\delta_{\text{max}}\), \(\eta_{\text{min}}\), and \(\eta_{\text{max}}\) denote the respective lower and upper bounds of \(\delta_i\) and \(\eta_i\). The parameters \(k_\delta > 0\) and \(k_\eta > 0\) regulate the rate of change.  

Instead of a fixed threshold, the adaptive threshold \(E_{\text{th}, i}\) is dynamically adjusted based on the energy levels of an agent's neighbors, ensuring that an agent gradually modifies its behavior when a sufficient number of its neighbors experience energy depletion. The adaptive threshold is defined as follows,
\[
E_{\text{th}, i} = E_{\text{th}} - \mu_i(E_{\text{th}} - \min_{j \in \mathcal{N}_i} E_j) ,
\]
with
\[
\mu_i=\frac{\sum_{j \in \mathcal{N}_i} \mathbb{I}(E_j < E_{\text{th}})}{|\mathcal{N}_i|},
\]
where \( \mathbb{I}(E_j < E_{\text{th}}) \) is an indicator function that is 1 if the energy of neighbor \( j \) is below the threshold \( E_{\text{th}} \), and 0 otherwise. The term \(\mu_i\) computes the fraction of neighbors with energy below the threshold. The adaptive threshold \( E_{\text{th}, i} \) is then smoothly adjusted by \(\mu_i\), blending between the \( E_{\text{th}} \) and the lowest energy in the neighborhood. This mechanism ensures that an agent starts adjusting \(\delta_i\) and \(\eta_i\) when at least one of its neighbors has an energy level below the predefined threshold \(E_{\text{th}}\), thereby enabling a smooth and coordinated transition between flocking and swarming behaviors.

Energy levels below \(E_\text{th}\) drive flocking behavior, whereas energy levels above \(E_\text{th}\) induce swarming behavior. When \(E_i\) remains above \(E_\text{th}\), \(\delta_i\) and \(\eta_i\) remain near their maximum values (i.e., \(\delta_{\text{max}}\) and \(\eta_{\text{max}}\)), leading to swarming behavior as agents are encouraged to spread out and explore. On the contrary, as \(E_i\) decreases and drops below \(E_\text{th}\), \(\delta_i\) and \(\eta_i\) approach their minimum values (i.e., \(\delta_{\text{min}}\) and \(\eta_{\text{min}}\)), driving flocking behavior as agents align their velocities and maintain close formations for energy conservation and efficient navigation toward resources. 

Substituting adaptive \(\delta_i\) and \(\eta_i\) in~\eqref{eq:adaptive_d} and~\eqref{eq:adaptive_e} into \(\psi(\lVert \mathbf{p}_j - \mathbf{p}_i \rVert)\) and \(\phi(\lVert \mathbf{v}_j - \mathbf{v}_i \rVert)\), the model in~\eqref{eq:model} becomes indirectly dependent on \(E_i\). The proposed energy-driven cognitive model inherently facilitates seamless transitions between swarming and flocking as a function of fluctuating energy levels. Agents modulate their behavioral strategies by dynamically balancing the competing demands of formation maintenance and environmental exploration, informed by their internal energy states. Within the domain of cognitive robotics, such mechanisms are fundamental to enabling adaptive autonomy, wherein agents make context-sensitive decisions by integrating external stimuli—such as the positions and velocities of neighboring agents—with internal states, including energy reserves and fatigue. This integrative decision-making framework enhances the system's ability to respond to environmental variability while maintaining coherence and task efficiency across diverse operational scenarios.

\begin{remark}
The parameters $c_{i1}$ and $c_{i2}$ control the energy depletion rate of agent $i$, thus influencing the timing of the cognitive behavior shift. Specifically, a higher $c_{i1}$ models faster energy drain during active maneuvering, while a higher $c_{i2}$ captures higher idle power consumption. Consequently, agents with faster energy depletion induce earlier transitions from swarming to flocking at the collective level. This mapping provides a tunable interface between hardware characteristics (e.g., battery capacity, actuator efficiency) and behavior-level adaptation, ensuring the mechanism is applicable to a wide range of robotic systems.
\end{remark}

\subsection{Role-Based Collaboration (RBC) Extension}\label{sec:RBC}
RBC is a computational methodology that uses roles as the fundamental mechanism for coordinating behaviors in multi-agent systems~\cite{Zhu2010,Zhu2015SMC}. Within this framework, agents assume roles that guide their interactions, task execution, and cooperation strategies. The RBC process typically involves role negotiation, qualification evaluation, and dynamic role assignment. It emphasizes structured planning to determine which agent should fulfill a particular role to achieve a shared objective. The roles can be both task-driven (assigned to accomplish objectives) and behavioral descriptors (inferred from emergent patterns). The E-CARGO model (Environments, Classes, Agents, Roles, Groups, and Objects) provides a formal ontology for collaboration~\cite{Zhu2012,Zhu2015RBC}. RBC/E-CARGO has emerged as a growing paradigm for endowing multi-agent systems with structured role assignments and negotiation processes, offering a robust foundation for adaptive coordination in autonomous agent collectives.

The proposed cognitive model currently modulates behavior such as swarming and flocking continuously, based on the energy levels of the agents and the interaction parameters. This continuous adaptation can be mapped to an RBC-style framework either by interpreting emergent behavioral states as latent roles or by explicitly introducing discrete roles that influence the same parameters. Several illustrative mappings are described in the following.

\begin{enumerate}
\item Energy-based roles: Agents can be categorized into role types according to their energy levels. For instance, high-energy agents may adopt an ‘explorer’ or ‘swarm’ role characterized by rapid movement and dispersion, while low-energy agents may assume a ‘flocker’ or ‘conserver’ role, prioritizing cohesion and energy-efficient alignment. In RBC terms, such roles correspond to different objective functions or behavioral constraints, and agents transition between roles in response to changes in their internal energy state.

\item Velocity and spatial roles: Roles may also be defined based on motion patterns and spatial positioning. For example, ‘leader’ roles might be assigned to agents at the forefront of a group trajectory, whereas ‘peripheral’ or ‘vortex’ roles may describe agents moving tangentially or encircling a region. Classification may be based on velocity coherence or spatial distribution, with corresponding interaction rules adapted to each role (e.g., modified alignment weight and attraction/repulsion range through altering kinetic and spatial offsets).

\item Explicit role augmentation: A more direct approach is to augment each agent’s state with a discrete 'role variable'. This variable would then govern the agent’s decision rules. For example, agents could be explicitly assigned roles such as 'flocker', 'swarmer', or 'vortexer', each role prescribing a particular form of the interactions (e.g., different kinetic and spatial offset ranges). Role assignment could be static (pre-assigned) or dynamic (using negotiation or optimization routines from RBC/E-CARGO literature). In an RBC/E-CARGO extension, we could treat each role as a separate Agent Class with its own Role object in the E-CARGO ontology and use established group role assignment (GRA) algorithms to allocate roles based on energy, position, or other agent-specific metrics.

\end{enumerate}

This role-based interpretation provides an alternative perspective on our continuous modulation scheme. For instance, the smooth transition from swarming to flocking can be viewed as an implicit shift in agent roles. Embedding such transitions into the RBC/E-CARGO framework offers a path toward formal role coordination, which is proposed as a direction for future investigation.

\section{Global Stability Analysis}\label{sec:global}

\subsection{Graph-theoretic representation}
The inter-agent interaction topology can be modeled using a directed graph $\mathcal{G}(\mathcal{V}, \mathcal{E})$, where each node in $\mathcal{V}$ represents an agent, and each directed edge $(j, i) \in \mathcal{E}$ denotes that agent $i$ receives information from agent $j$ and $i \neq j$.

The set of incoming neighbors of agent $i$ is defined as $\mathcal{N}_i^{\text{in}} = \{ j \in \mathcal{V} \mid (j, i) \in \mathcal{E} \}$. When interactions are determined by spatial proximity, this set can be expressed as $\mathcal{N}_i = \left\{ j \in \mathcal{V} \mid j \neq i,\ \lVert \mathbf{p}_j - \mathbf{p}_i \rVert \leq r_i \right\}$. 

A node is globally reachable if it can be reached via a directed path from every other node. A graph is connected if it contains at least one such node, referred to as the root of a spanning tree. A spanning tree is a subgraph that connects all nodes with exactly one directed path to the root. 

Let $\mathbf{D} \in \mathbb{R}^{|\mathcal{V}| \times |\mathcal{E}|}$ denote the incidence matrix of $\mathcal{G}$, where each column encodes the direction of an edge: entries are $1$ at the head node, $-1$ at the tail node, and $0$ otherwise. The graph Laplacian is then defined as $\mathbf{L} = \mathbf{D} \mathbf{D}^\top \in \mathbb{R}^{|\mathcal{V}| }$. For systems evolving in $m$-dimensional space, the extended Laplacian is given by $\mathbf{L} \otimes \mathbf{I}_m \in \mathbb{R}^{m|\mathcal{V}|}$ where $\otimes$ denotes the Kronecker product and $\mathbf{I}_m$ is the $m$-dimensional identity matrix~\cite{Olfati-Saber}. $\mathbf{L} \otimes \mathbf{I}_m$ is symmetric and positive semidefinite.

\subsection{Global interaction dynamics}

Define the global position and velocity vectors \(\mathbf{p}=[\mathbf{p}_1^\top,\cdots,\mathbf{p}_n^\top]^\top\in \mathbb{R}^{nm}\) and \(\mathbf{v}=[\mathbf{v}_1^\top,\cdots,\mathbf{v}_n^\top]^\top\in \mathbb{R}^{nm}\), respectively. Let \(\mathbf{p}_{ij}=(\frac{\delta_i |\mathcal{N}_i|}{\lVert \mathbf{p}_j-\mathbf{p}_i\rVert})^\alpha(\mathbf{p}_j-\mathbf{p}_i)\in \mathbb{R}^{m}\) and \(\mathbf{v}_{ij}=(\frac{\eta_i}{|\mathcal{N}_i|\lVert \mathbf{v}_j-\mathbf{v}_i\rVert})^\beta(\mathbf{v}_j-\mathbf{v}_i)\in \mathbb{R}^{m}\). Define \(\mathbf{q}=[\cdots,\mathbf{p}_{ij}^\top,\cdots]^\top\in\mathbb{R}^{m|\mathcal{E}|}\) and \(\tilde{\mathbf{q}}=[\cdots,\mathbf{v}_{ij}^\top,\cdots]^\top\in\mathbb{R}^{m|\mathcal{E}|}\) where the order of the elements corresponds to the order of the edges in \(\mathcal{E}\). We have 
\begin{align}
    &\sum_{j\in\mathcal{N}_i}(\mathbf{p}_j-\mathbf{p}_i-\mathbf{p}_{ij})=-(\mathbf{d}_i\mathbf{D}^\top\otimes \mathbf{I}_m) \mathbf{p}+(\mathbf{d}_i\otimes \mathbf{I}_m)\mathbf{q},\label{eq:sum_p} \\
    &\sum_{j\in\mathcal{N}_i}(\mathbf{v}_j-\mathbf{v}_i-\mathbf{v}_{ij})=-(\mathbf{d}_i\mathbf{D}^\top\otimes \mathbf{I}_m) \mathbf{v}+(\mathbf{d}_i\otimes \mathbf{I}_m)\tilde{\mathbf{q}},\label{eq:sum_v}
\end{align}
where $\mathbf{d}_i$ denote the $i$-th row in $\mathbf{D}$.

Using~\eqref{eq:sum_p} and~\eqref{eq:sum_v}, the global interaction dynamics are obtained as
\begin{align}\label{eq:model-g}
    \frac{d\mathbf{p}}{dt}=&\mathbf{v}, \notag\\
    \frac{d\mathbf{v}}{dt}=&-(\begin{bmatrix}
        \mathbf{d}_1 \\ \vdots \\ \mathbf{d}_n
    \end{bmatrix}\mathbf{D}^\top\otimes \mathbf{I}_m) \mathbf{p}-(\begin{bmatrix}
        \mathbf{d}_1 \\ \vdots \\ \mathbf{d}_n
    \end{bmatrix}\mathbf{D}^\top\otimes \mathbf{I}_m) \mathbf{v}\notag\\&+(\begin{bmatrix}
      \mathbf{d}_1\\ \vdots \\ \mathbf{d}_n
    \end{bmatrix}\otimes \mathbf{I}_m)\mathbf{q}+(\begin{bmatrix}
      \mathbf{d}_1\\ \vdots \\ \mathbf{d}_n
    \end{bmatrix}\otimes \mathbf{I}_m)\tilde{\mathbf{q}}. 
\end{align}

Note that 
\[
(\begin{bmatrix}
      \mathbf{d}_1 \\ \vdots \\ \mathbf{d}_n
    \end{bmatrix}\otimes \mathbf{I}_m)\mathbf{q}=(\begin{bmatrix}
    \mathbf{d}_1\mathbf{W}_1 \\ \vdots \\ \mathbf{d}_n\mathbf{W}_n
\end{bmatrix}\mathbf{D}^\top\otimes \mathbf{I}_m) \mathbf{p},
\]
and
\[
(\begin{bmatrix}
      \mathbf{d}_1 \\ \vdots \\ \mathbf{d}_n
    \end{bmatrix}\otimes \mathbf{I}_m)\tilde{\mathbf{q}}=(\begin{bmatrix}
    \mathbf{d}_1\tilde{\mathbf{W}}_1 \\ \vdots \\ \mathbf{d}_n\tilde{\mathbf{W}}_n
\end{bmatrix}\mathbf{D}^\top\otimes \mathbf{I}_m) \mathbf{v},
\]
 where \( \mathbf{W}_i = \mathrm{diag}(w_i) \in \mathbb{R}^{|\mathcal{E}|} \) and \( w_i = \mathrm{col}(w_e) \) for all \( e \in \mathcal{E} \), \( \tilde{\mathbf{W}}_i = \mathrm{diag}(\tilde{w}_i) \in \mathbb{R}^{|\mathcal{E}|} \) and \( \tilde{w}_i = \mathrm{col}(\tilde{w}_e) \) for all \( e \in \mathcal{E} \), with
\[
w_e = \begin{cases} 
(\frac{\delta_i |\mathcal{N}_i|}{\lVert \mathbf{p}_j-\mathbf{p}_i\rVert})^\alpha, & \text{if } e = (i, j) \text{ and } j \in \mathcal{N}_i, \\
0, & \text{otherwise},
\end{cases}
\]
and
\[
\tilde{w}_e = \begin{cases} 
(\frac{\eta_i }{|\mathcal{N}_i|\lVert \mathbf{v}_j-\mathbf{v}_i\rVert})^{\beta}, & \text{if } e = (i, j) \text{ and } j \in \mathcal{N}_i, \\
0, & \text{otherwise},
\end{cases}
\]
and \( \mathrm{col}\) denotes the column vector. 

Thus, the global interaction dynamics in~\eqref{eq:model-g} are simplified as to
\begin{align}\label{eq:model-g2}
    \frac{d\mathbf{p}}{dt}=&\mathbf{v}, \notag\\
    \frac{d\mathbf{v}}{dt}=&-(\begin{bmatrix}
        \mathbf{d}_1\bar{\mathbf{W}}_1 \\ \vdots \\ \mathbf{d}_n\bar{\mathbf{W}}_n
    \end{bmatrix}\mathbf{D}^\top\otimes \mathbf{I}_m) \mathbf{p}-(\begin{bmatrix}
        \mathbf{d}_1\hat{\mathbf{W}}_1 \\ \vdots \\ \mathbf{d}_n\hat{\mathbf{W}}_n
    \end{bmatrix}\mathbf{D}^\top\otimes \mathbf{I}_m) \mathbf{v},
\end{align}
where \(\bar{\mathbf{W}}_i=(\mathbf{I}-\mathbf{W}_i)\) and \(\hat{\mathbf{W}}_i=(\mathbf{I}-\tilde{\mathbf{W}}_i)\).

Denote
\[
\bar{\mathbf{D}} = \begin{bmatrix}
         \mathbf{d}_1 \bar{\mathbf{W}}_1 \\ \vdots \\  \mathbf{d}_n \bar{\mathbf{W}}_n
    \end{bmatrix} \otimes \mathbf{I}_m, \quad \hat{\mathbf{D}} = \begin{bmatrix}
         \mathbf{d}_1 \hat{\mathbf{W}}_1 \\ \vdots \\ \mathbf{d}_n \hat{\mathbf{W}}_n
    \end{bmatrix} \otimes \mathbf{I}_m,
\]
where both matrices have a structure similar to the incidence matrix with elements scaled by weights embedded in the diagonals of \(\bar{\mathbf{W}}_i\) and \(\hat{\mathbf{W}}_i\) for each node \(i\). 

Finally, the compact form of the global interaction dynamics is given by
\begin{align}\label{eq:global}
     \frac{d\mathbf{p}}{dt}=&\mathbf{v}, \notag\\
    \frac{d\mathbf{v}}{dt}=&-\bar{\mathbf{D}}(\mathbf{D}^\top\otimes \mathbf{I}_m)\mathbf{p} -\hat{\mathbf{D}}(\mathbf{D}^\top\otimes \mathbf{I}_m) \mathbf{v}.
\end{align}

\subsection{Asymptotic stability}

The interplay of repulsive and attractive forces within the proposed interaction framework drives agents toward steady-state positions and velocities. This corresponds to the asymptotic convergence of all edge-based interaction errors to zero, formally expressed as
\begin{equation} \label{eq:requirement}
\mathbf{p}_j-\mathbf{p}_i-\mathbf{p}_{ij}\rightarrow 0,\quad\mathbf{v}_j-\mathbf{v}_i-\mathbf{v}_{ij}\rightarrow 0,
\end{equation}
for all \((j,i)\in\mathcal{E}\) as \(t\rightarrow\infty\).

\begin{theorem}\label{theorem:stability}
Consider a MAS with the global interaction dynamics in~\eqref{eq:global}. Let the corresponding interaction graph \(\mathcal{G}\) form a spanning tree. Then, the edge interaction errors in~\eqref{eq:requirement} asymptotically converge to zero as \(t \to \infty\).
\end{theorem}

\begin{proof}
Let $\mathbf{e}=(-\mathbf{D}^\top\otimes \mathbf{I}_m) \mathbf{p}$, so the edge dynamics are
$$
\frac{d\mathbf{e}}{dt} = (-\mathbf{D}^\top\otimes \mathbf{I}_m)\mathbf{v}, \quad
\frac{d^2\mathbf{e}}{dt^2} = (-\mathbf{D}^\top\otimes \mathbf{I}_m) \frac{d\mathbf{v}}{dt}.
$$
 
Substitute $\frac{d\mathbf{v}}{dt}$ from~\eqref{eq:global},
$$
\frac{d^2\mathbf{e}}{dt^2} = -(\mathbf{D}^\top\otimes \mathbf{I}_m)\bar{\mathbf{D}}\mathbf{e} - (\mathbf{D}^\top\otimes \mathbf{I}_m)\hat{\mathbf{D}}\frac{d\mathbf{e}}{dt}.
$$

Let $\mathbf{A} = (\mathbf{D}^\top\otimes \mathbf{I}_m)\hat{\mathbf{D}}$, and $\mathbf{B} = (\mathbf{D}^\top\otimes \mathbf{I}_m)\bar{\mathbf{D}}$, then,
$$
\frac{d^2\mathbf{e}}{dt^2} + \mathbf{A}\frac{d\mathbf{e}}{dt} + \mathbf{B}\mathbf{e} = 0.
$$

Define the candidate Lyapunov function,
$$
V(\mathbf{e}, \frac{d\mathbf{e}}{dt}) = \frac{1}{2} (\frac{d\mathbf{e}}{dt})^\top \frac{d\mathbf{e}}{dt} + \frac{1}{2} \mathbf{e}^\top \mathbf{B} \mathbf{e},
$$
which is positive semi-definite because Laplacian-alike $\mathbf{B}$ is positive semidefinite. Taking the time derivative along system trajectories gives,
$$
\frac{d\mathbf{V}}{dt} = (\frac{d\mathbf{e}}{dt})^\top \frac{d^2\mathbf{e}}{dt^2} + \mathbf{e}^\top \mathbf{B} \frac{d\mathbf{e}}{dt} =(\frac{d\mathbf{e}}{dt})^\top (-\mathbf{A}\frac{d\mathbf{e}}{dt} - \mathbf{B}\mathbf{e}) + \mathbf{e}^\top \mathbf{B} \frac{d\mathbf{e}}{dt}.
$$

Noting that $\mathbf{e}^\top \mathbf{B} \frac{d\mathbf{e}}{dt} = (\frac{d\mathbf{e}}{dt})^\top \mathbf{B} \mathbf{e}$, we simplify,
$$
\frac{d\mathbf{V}}{dt} = -(\frac{d\mathbf{e}}{dt})^\top \mathbf{A}\frac{d\mathbf{e}}{dt}.
$$

Since Laplacian-alike $\mathbf{A}$ is positive semidefinite, it follows that $\frac{d\mathbf{V}}{dt}\leq 0$, establishing stability. Moreover, the largest invariant set where $\frac{d\mathbf{V}}{dt} = 0$ corresponds to $\dot{\mathbf{e}} \in \ker(\mathbf{A})$. If $\mathbf{A}$ is positive definite (e.g., for at least one strictly positive $\eta_i$ per edge and spanning tree topology), then $\dot{\mathbf{e}} \to 0$. From the dynamics, this implies $\frac{d^2\mathbf{e}}{dt^2} \to 0$ and thus $\mathbf{e}^\top \mathbf{B} \mathbf{e} \to 0$, i.e., $\mathbf{e} \to 0$.

Therefore, the edge system is globally asymptotically stable: $\mathbf{e}\to 0$ and $\frac{d\mathbf{e}}{dt}\to 0$ as $t\to\infty$.
\end{proof}

Then, all agents asymptotically achieve their desired relative positions and velocities defined by the offset vectors \(\mathbf{p}_{ij}\) and \(\mathbf{v}_{ij}\). The analysis assumes a fixed topology for tractability, while the model inherently features dynamic neighbor selection, resulting in a time-varying topology. This is addressed by applying the fixed-topology stability result at each time step, treating the current state as a new initial condition, provided connectivity is preserved. Hence, convergence to the desired relative positions and velocities remains valid.

\section{Simulation Results and Insights}\label{sec:sim}
\subsection{Minimal model}
First, we validate the model by showing that it generates various motion patterns under different settings, as illustrated in Fig.~\ref{fig:col_beh}. Emerging collective behaviors are flocking (Fig.~\ref{fig:col_beh}(\subref{fig:col_beh:f})), vortexing (Fig.~\ref{fig:col_beh}(\subref{fig:col_beh:v1})), and swarming (Fig.~\ref{fig:col_beh}(\subref{fig:col_beh:s})). 
\begin{figure*}[htbp]
    \centering
    \begin{subfigure}{0.3\linewidth}
        \centering
        \includegraphics[width=\linewidth]{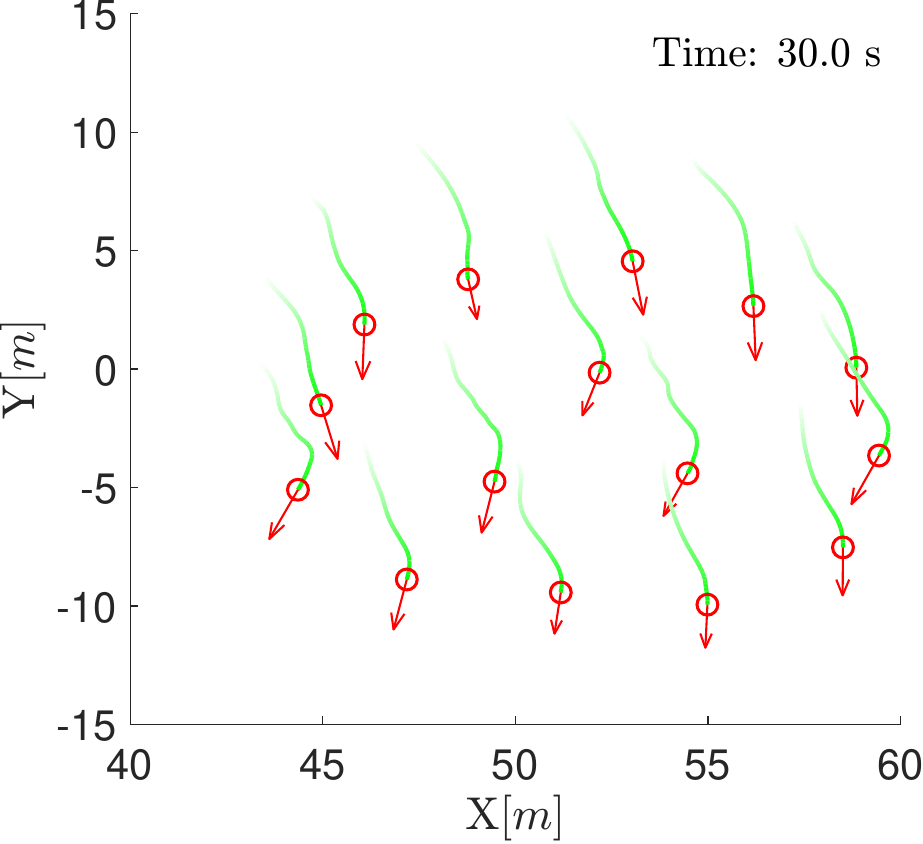}
        \subcaption[]{flocking} \label{fig:col_beh:f}
    \end{subfigure}
    \begin{subfigure}{0.3\linewidth}
        \centering
        \includegraphics[width=\linewidth]{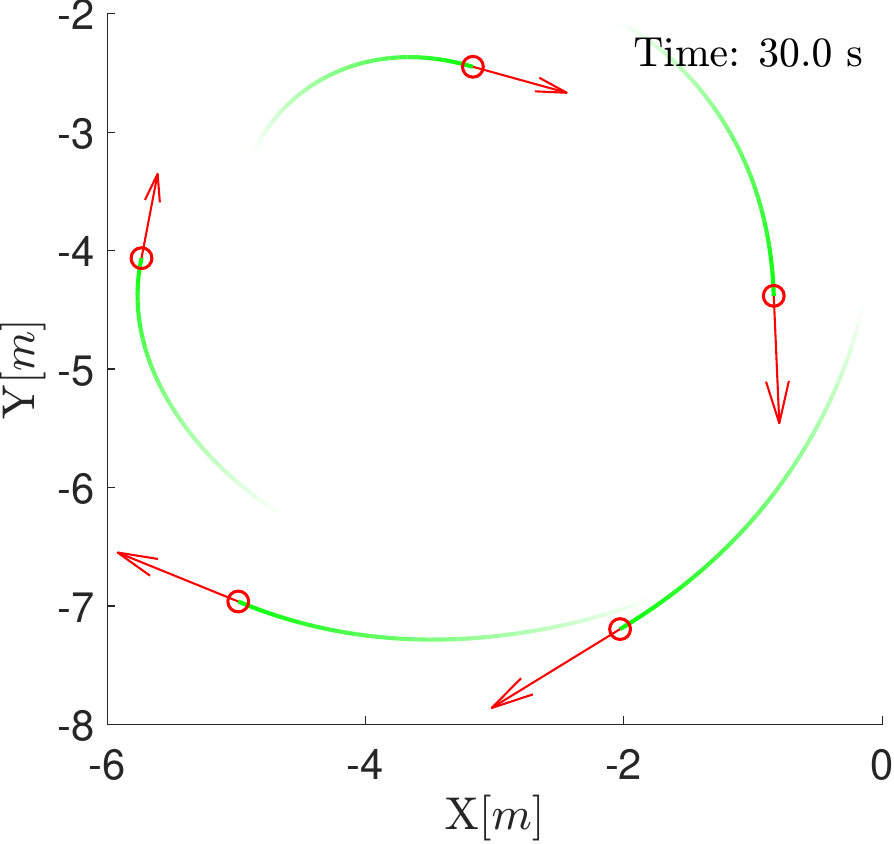}
        \subcaption[]{vortexing} \label{fig:col_beh:v1}
    \end{subfigure}
    \begin{subfigure}{0.3\linewidth}
        \centering
        \includegraphics[width=\linewidth]{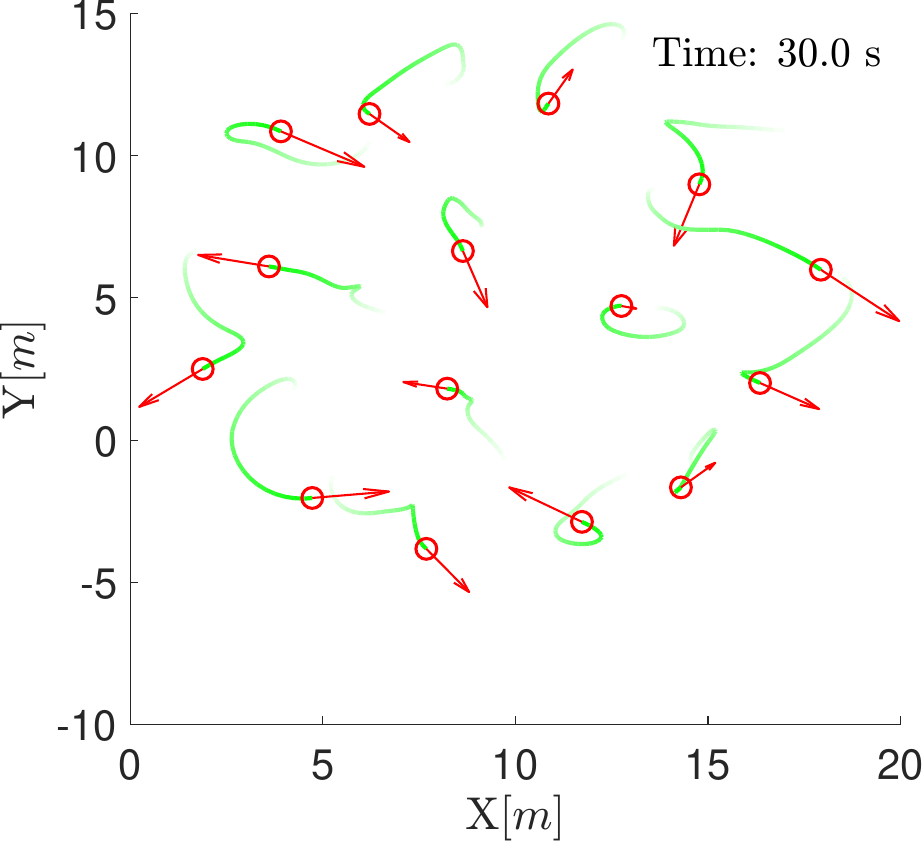}
        \subcaption[]{swarming} \label{fig:col_beh:s}
    \end{subfigure}
\caption{The model generates motion patterns that exhibit emergent collective behaviors of flocking, vortexing, and swarming, as depicted in (a) to (c), respectively, with \(m=2,\alpha=2,\beta=1\). These behaviors were observed under the following model setting: (a) \(n = 15, \eta_i =\eta_{\text{max}}=\eta_{\text{min}} = 3\), (b) \(n = 5, \eta_i =\eta_{\text{max}}=\eta_{\text{min}} = 6\), and (c) \(n = 15, \eta_i=\eta_{\text{max}}=\eta_{\text{min}} = 12\), with \(\delta_i =\delta_{\text{max}}=\delta_{\text{min}} = 1\), \(r_i=10\,\mathrm{m}\), and total time \(t = 30\,\mathrm{s}\) for all cases. The agents' positions and velocities were initialized randomly within the range $(0,10)\,\mathrm{m}$ and \((-1,1)\,\mathrm{m/s}\). The maximum speed and the time required for each agent to reach it are set to \(v_i^{\text{max}} = 5\,\mathrm{m/s}\) and \(t_i^{v_\text{max}} = 1\,\mathrm{s}\), respectively. The red arrow represents the velocity vector of the agent, with its length denoting the agent's speed, and the green trajectory indicates the agent's path over the last \(3\,\mathrm{s}\). }
    \label{fig:col_beh}
\end{figure*}

To assess the coherence and coordination of these behaviors, we examine the temporal evolution of average edge-based interaction errors for each agent, defined as
\[
\frac{1}{|\mathcal{N}_i|} \sum_{j\in\mathcal{N}_i} (\mathbf{p}_j - \mathbf{p}_i - \mathbf{p}_{ij}),\quad\frac{1}{|\mathcal{N}_i|} \sum_{j\in\mathcal{N}_i} (\mathbf{v}_j - \mathbf{v}_i - \mathbf{v}_{ij}),
\]
and present the results in Fig.~\ref{fig:edge}. Figs.~\ref{fig:edge}(\subref{fig:edge:px:f})–(\subref{fig:edge:vy:f}) correspond to the flocking scenario, showing the evolution of position and velocity errors along the $x$- and $y$-axes. The plots exhibit rapid convergence of position errors, indicative of cohesive motion. Velocity convergence, however, remains partial due to the nonzero kinetic offset $\eta_i = 3$, which permits controlled velocity divergence among neighboring agents. Figs.~\ref{fig:edge}(\subref{fig:edge:px:v})–(\subref{fig:edge:vy:v}) represent the vortexing scenario. In this case, both position and velocity errors rapidly diminish once the vortex pattern stabilizes, confirming the emergence of coordinated rotational motion. Finally, Figs.~\ref{fig:edge}(\subref{fig:edge:px:s})–(\subref{fig:edge:vy:s}) depict the swarming behavior. While the position errors converge swiftly—ensuring group cohesion—the velocity errors remain relatively large, consistent with the high kinetic offset $\eta_i = 12$, which facilitates significant velocity heterogeneity and promotes exploratory dynamics.
\begin{figure}[htbp]
    \centering
    \begin{subfigure}{0.24\linewidth}
        \centering
        \includegraphics[width=\linewidth]{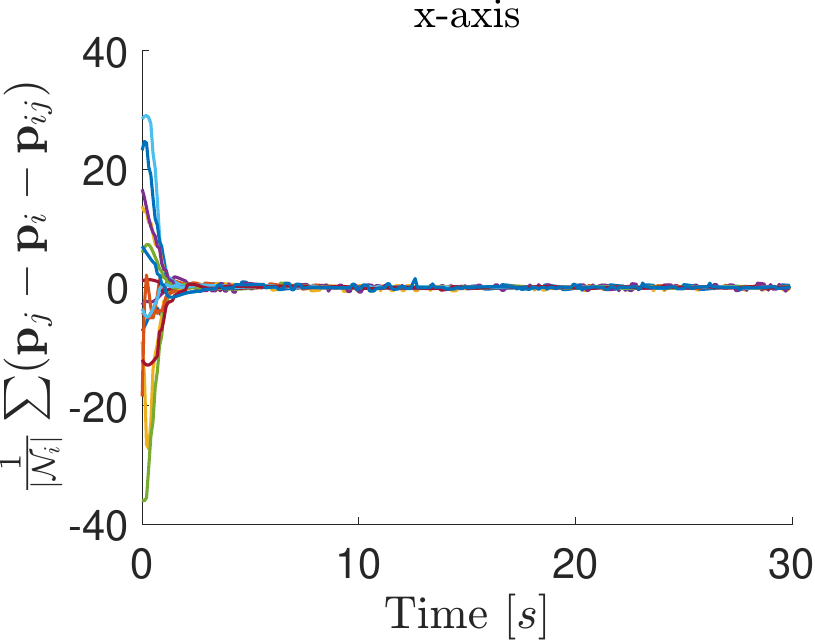}
        \subcaption[]{} \label{fig:edge:px:f}
    \end{subfigure}
    \begin{subfigure}{0.24\linewidth}
        \centering
        \includegraphics[width=\linewidth]{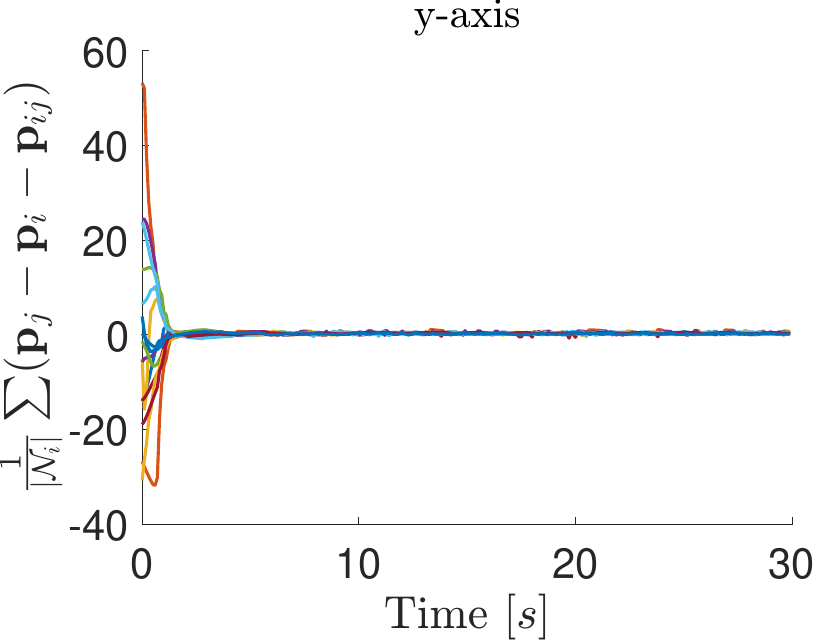}
        \subcaption[]{} \label{fig:edge:py:f}
    \end{subfigure}
    \begin{subfigure}{0.24\linewidth}
        \centering
        \includegraphics[width=\linewidth]{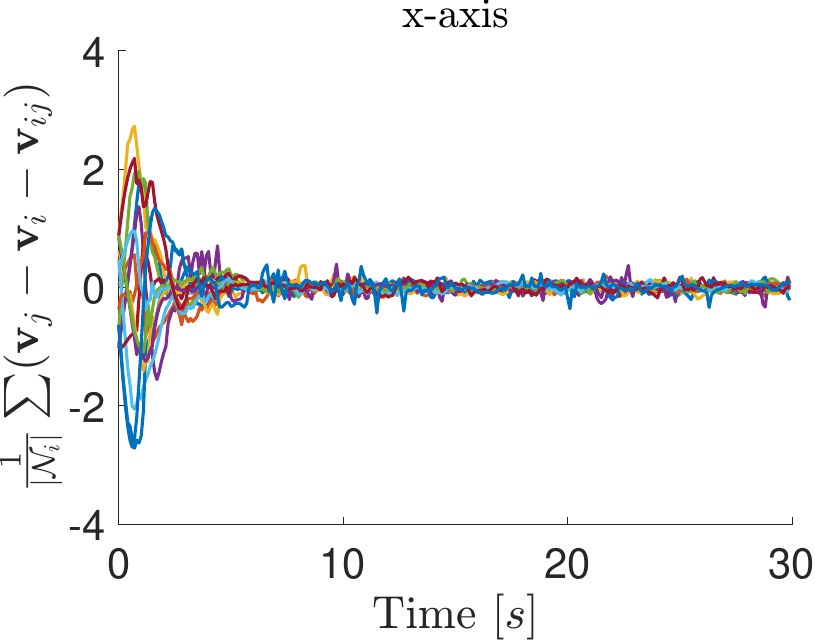}
        \subcaption[]{} \label{fig:edge:vx:f}
    \end{subfigure}
    \begin{subfigure}{0.24\linewidth}
        \centering
        \includegraphics[width=\linewidth]{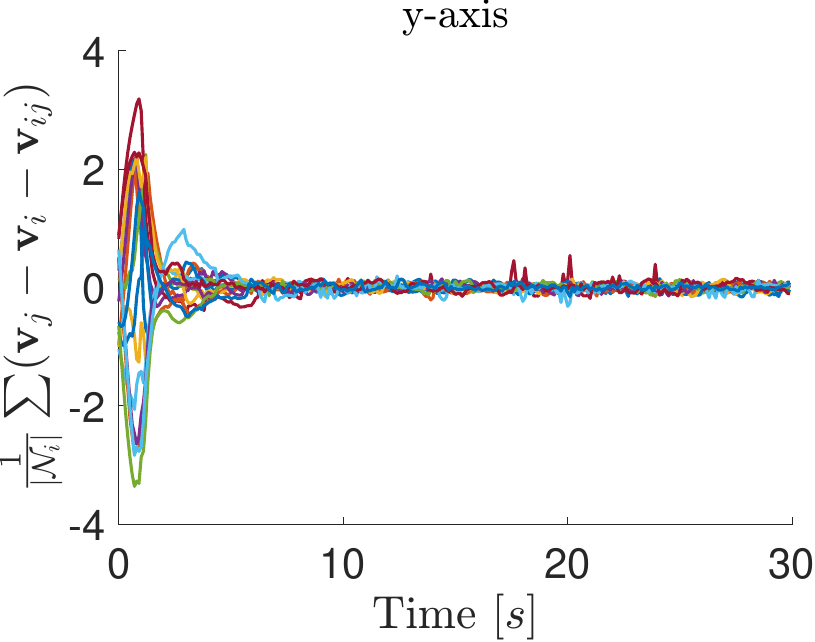}
        \subcaption[]{} \label{fig:edge:vy:f}
    \end{subfigure}
    \begin{subfigure}{0.24\linewidth}
        \centering
        \includegraphics[width=\linewidth]{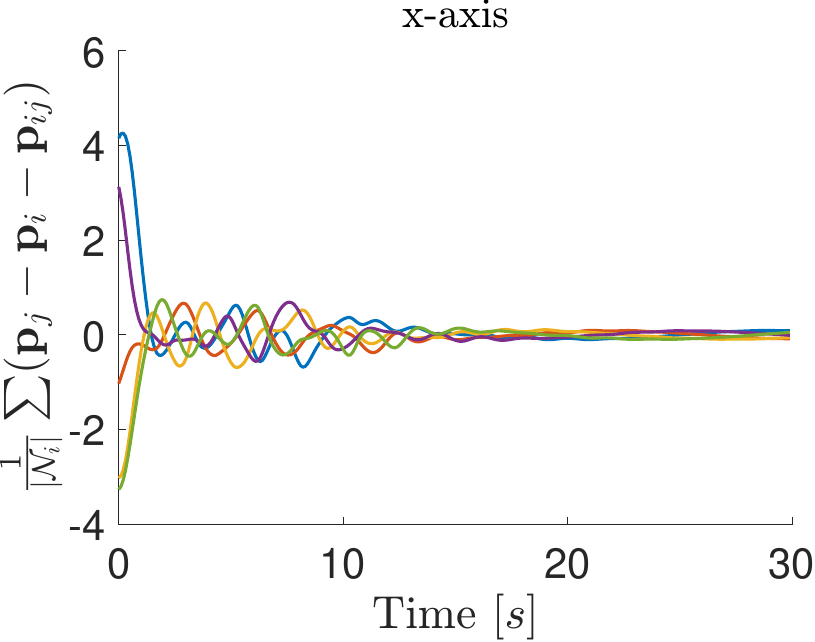}
        \subcaption[]{} \label{fig:edge:px:v}
    \end{subfigure}
    \begin{subfigure}{0.24\linewidth}
        \centering
        \includegraphics[width=\linewidth]{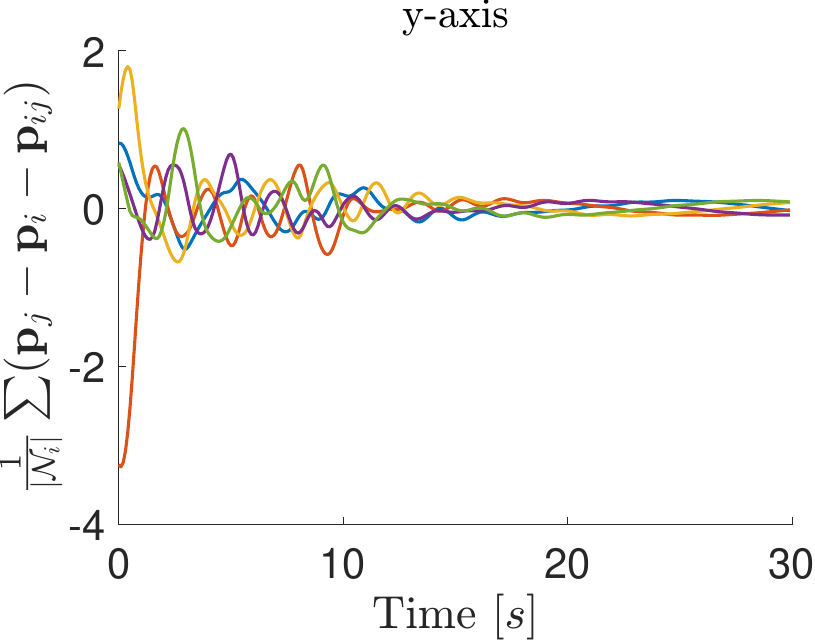}
        \subcaption[]{} \label{fig:edge:py:v}
    \end{subfigure}
    \begin{subfigure}{0.24\linewidth}
        \centering
        \includegraphics[width=\linewidth]{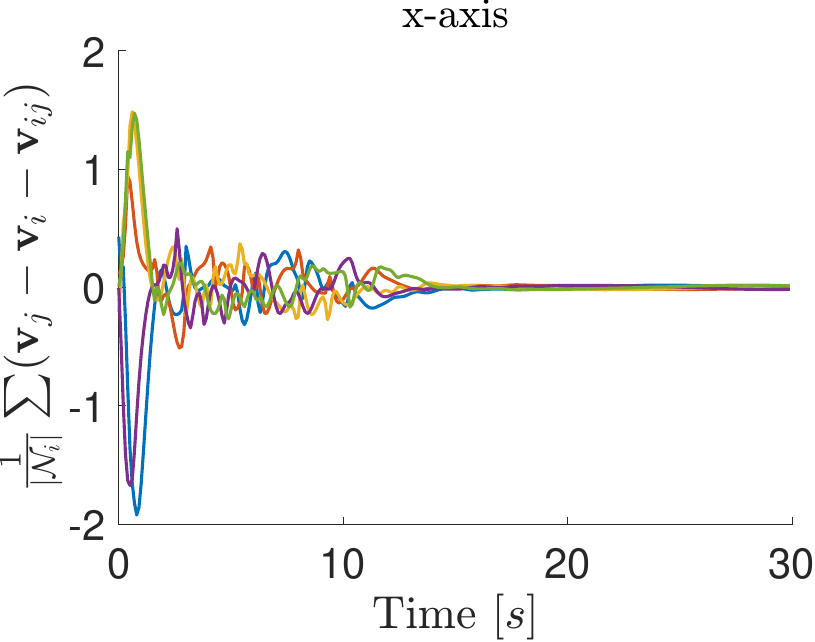}
        \subcaption[]{} \label{fig:edge:vx:v}
    \end{subfigure}
    \begin{subfigure}{0.24\linewidth}
        \centering
        \includegraphics[width=\linewidth]{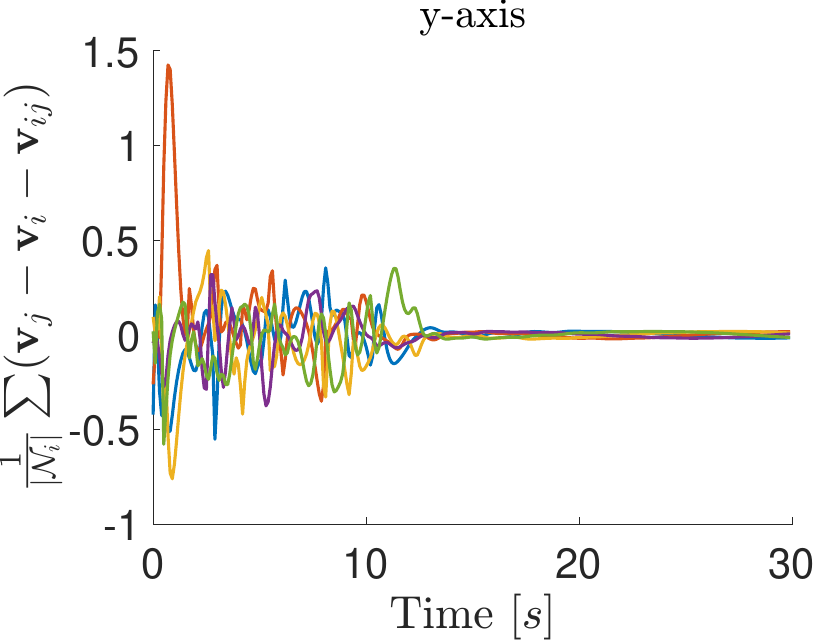}
        \subcaption[]{} \label{fig:edge:vy:v}
    \end{subfigure}
    \begin{subfigure}{0.24\linewidth}
        \centering
        \includegraphics[width=\linewidth]{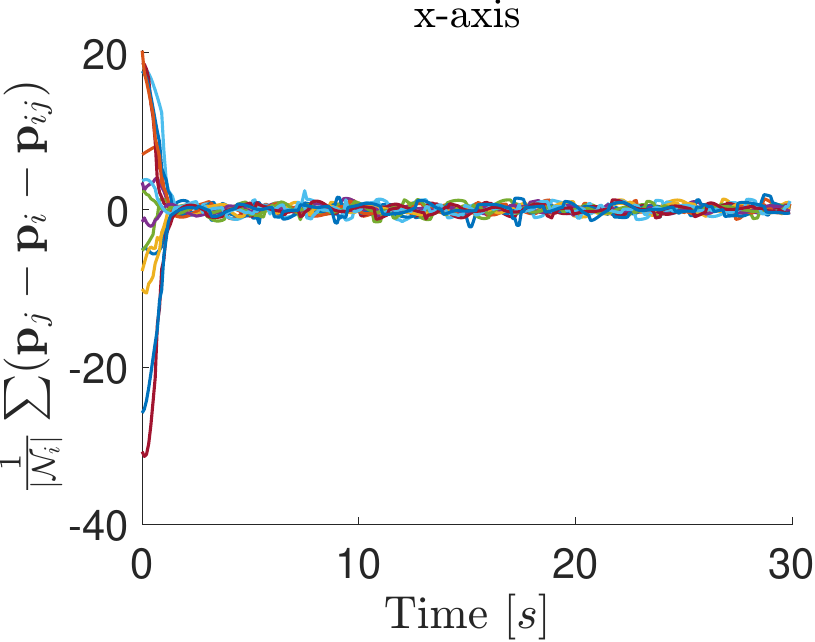}
        \subcaption[]{} \label{fig:edge:px:s}
    \end{subfigure}
    \begin{subfigure}{0.24\linewidth}
        \centering
        \includegraphics[width=\linewidth]{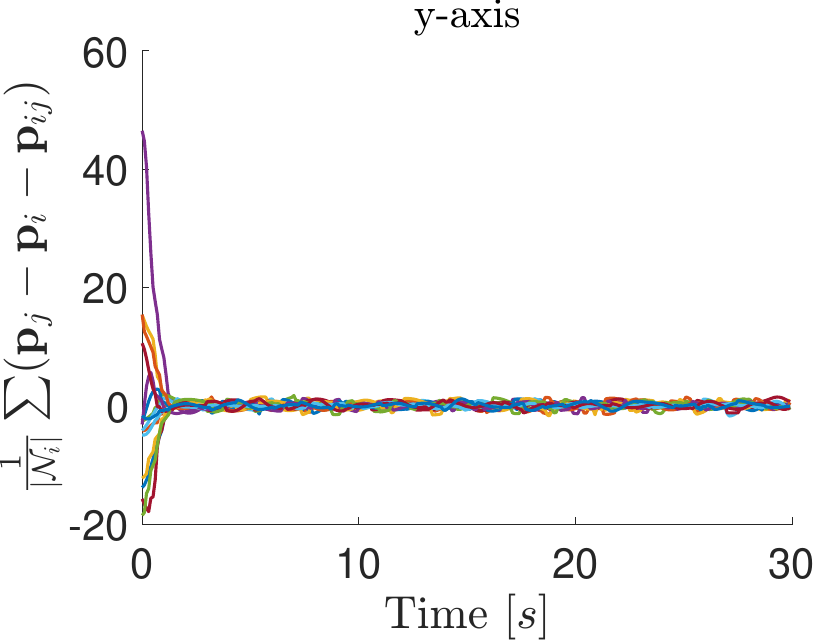}
        \subcaption[]{} \label{fig:edge:py:s}
    \end{subfigure}
    \begin{subfigure}{0.24\linewidth}
        \centering
        \includegraphics[width=\linewidth]{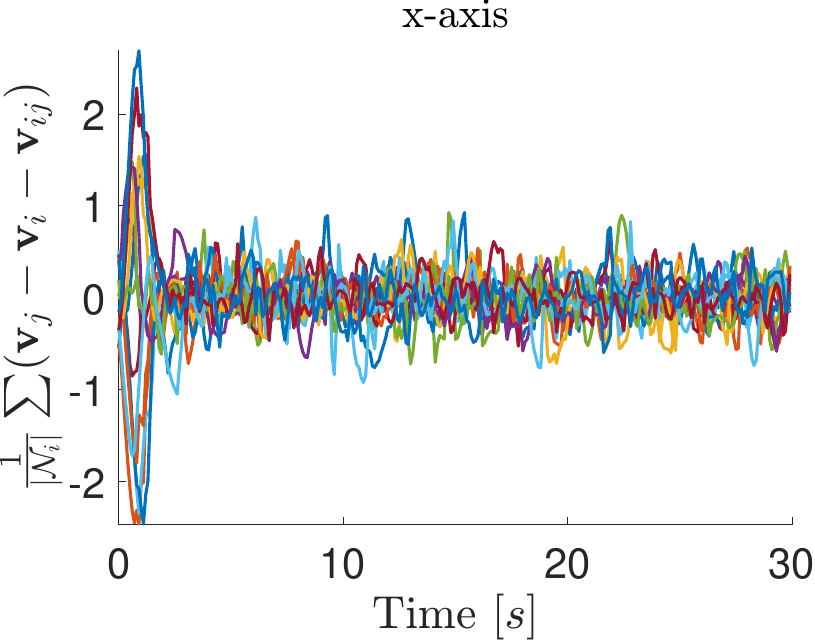}
        \subcaption[]{} \label{fig:edge:vx:s}
    \end{subfigure}
    \begin{subfigure}{0.24\linewidth}
        \centering
        \includegraphics[width=\linewidth]{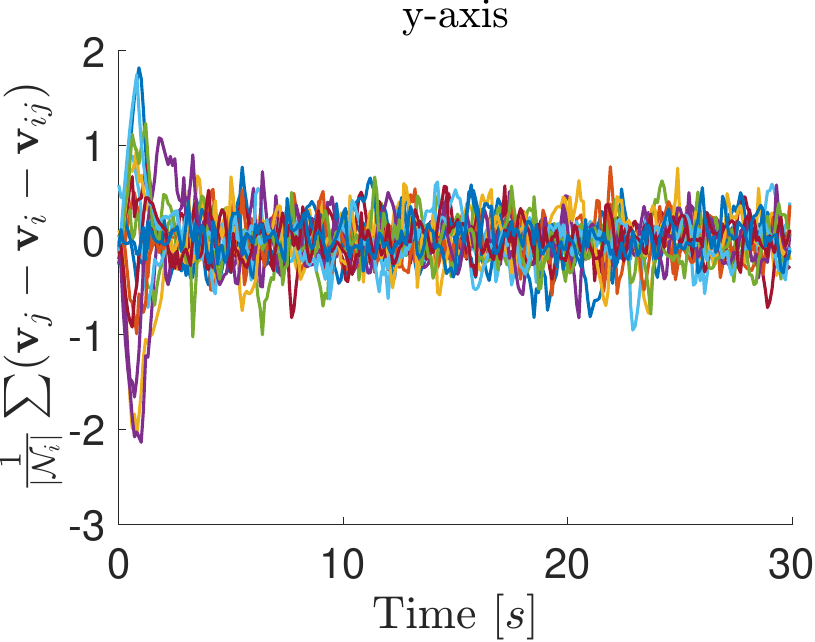}
        \subcaption[]{} \label{fig:edge:vy:s}
    \end{subfigure}
    \caption{Time evolution of average edge-based interaction errors for position and velocity in representative collective behaviors. (a–d) Flocking: position and velocity errors along the $x$- and $y$-axes, respectively, showing rapid convergence in position and partial convergence in velocity due to a moderate kinetic offset $\eta_i = 3$. (e–h) Vortexing: both position and velocity errors converge swiftly upon vortex formation, indicating coordinated rotational motion. (i–l) Swarming: position errors exhibit fast convergence, while velocity errors remain high due to a large kinetic offset $\eta_i = 12$, enabling diverse motion directions among agents.}
    \label{fig:edge}
\end{figure}

To systematically analyze the dynamics under varying conditions, we define several key indicators that capture behavioral transitions and their underlying causes. In particular, we expect to observe that transitions between these behaviors are closely linked to the emergence of alignment and misalignment in collective motion. We quantify the transition continuum using a metric that evaluates the overall alignment by computing the average cosine of the angles between all pairs of unit velocity vectors. The overall alignment metric \(h\) is defined as,  
\[
h = \text{avg}(\sum_{\forall(i,j)} \frac{\mathbf{v}_i^\top \mathbf{v}_j}{\|\mathbf{v}_i\| \| \mathbf{v}_j\|}),
\]
where \(\text{avg}(\cdot)\) denotes the average value, and the dot product returns the cosine of the angle between the unit velocity vectors. The maximum value of \(h = 1\) indicates perfect alignment, reflecting complete order. The flocking behavior emerges at higher values of \(h\). As the value of \(h\) decreases, the degree of disorder in the motion increases. In particular, \(h \approx -1\) signifies total misalignment with agents pointing in opposite directions and \(h \approx 0\) suggests random or orthogonal velocities. Vortexing behavior tends to arise when the number of agents is small and \(h\) takes negative values, whereas swarming typically occurs with larger agent populations and \(h\) values close to zero.

To evaluate whether the group maintains cohesion during collective motion, the radius of aggregation is defined as the maximum distance between any agent and the center of dispersion
\[
r_{\text{agg}} = \max_{i} \| \mathbf{p}_i - \frac{1}{n} \sum_{i=1}^{n} \mathbf{p}_i \|.
\]  

The typical motion patterns and behavioral phase transition diagram shown in Fig.~\ref{fig:analysis}(\subref{fig:analysis:phase}), the corresponding radii of aggregation in Fig.~\ref{fig:analysis}(\subref{fig:analysis:disp}), and the average speeds in Fig.~\ref{fig:analysis}(\subref{fig:analysis:speed}) provide an overview of our model for various population sizes \(n\). In particular, as \(n\) increases, the behavioral phase transition becomes smoother and eventually converges to random motion directions at the individual level. In summary, Fig.~\ref{fig:analysis} demonstrates that \(\eta_i\) induces the behavioral phase transition in collective motion, leading to the emergence of flocking, vortexing, and swarming behaviors. 
   
\begin{figure}[t]
    \centering
    \begin{subfigure}{0.32\linewidth}
        \centering
        \includegraphics[width=\linewidth]{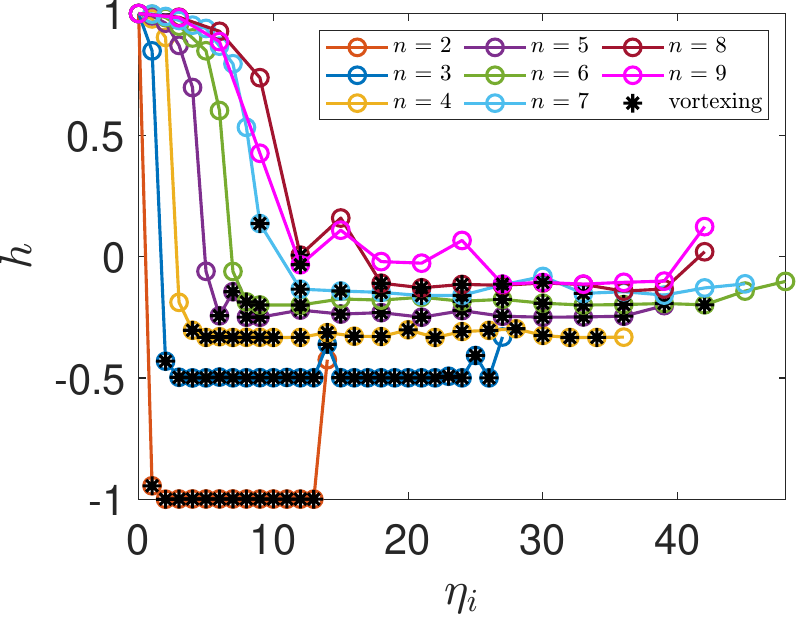}
    \end{subfigure}
    \begin{subfigure}{0.32\linewidth}
        \centering
        \includegraphics[width=\linewidth]{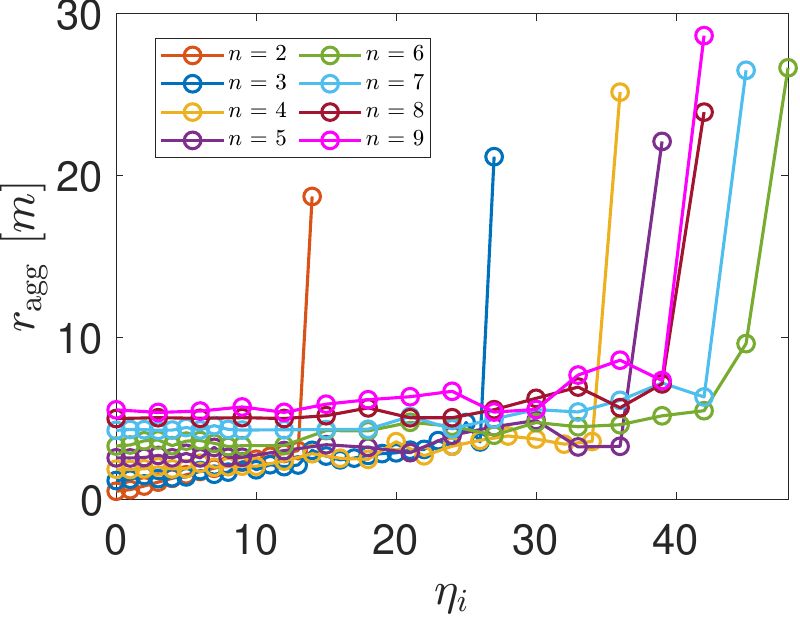}
    \end{subfigure}
     \begin{subfigure}{0.32\linewidth}
        \centering
        \includegraphics[width=\linewidth]{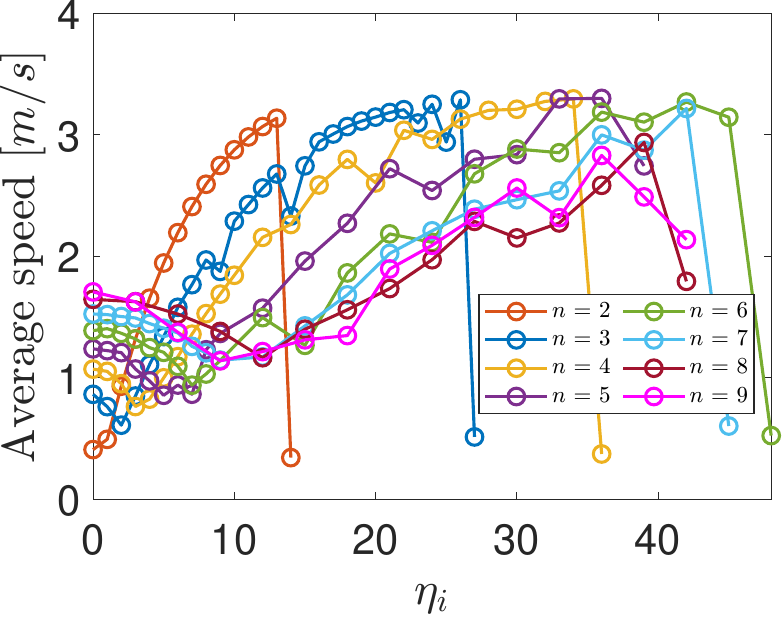}
    \end{subfigure}
    \vskip 0.5\baselineskip
    \begin{subfigure}{0.32\linewidth}
        \centering
        \includegraphics[width=\linewidth]{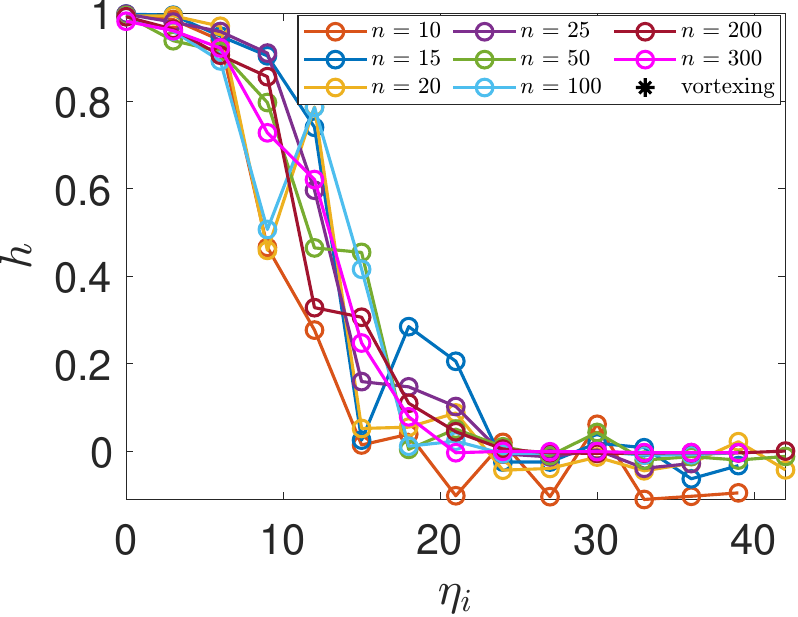}
        \subcaption[]{} \label{fig:analysis:phase}
    \end{subfigure}
    \begin{subfigure}{0.32\linewidth}
        \centering
        \includegraphics[width=\linewidth]{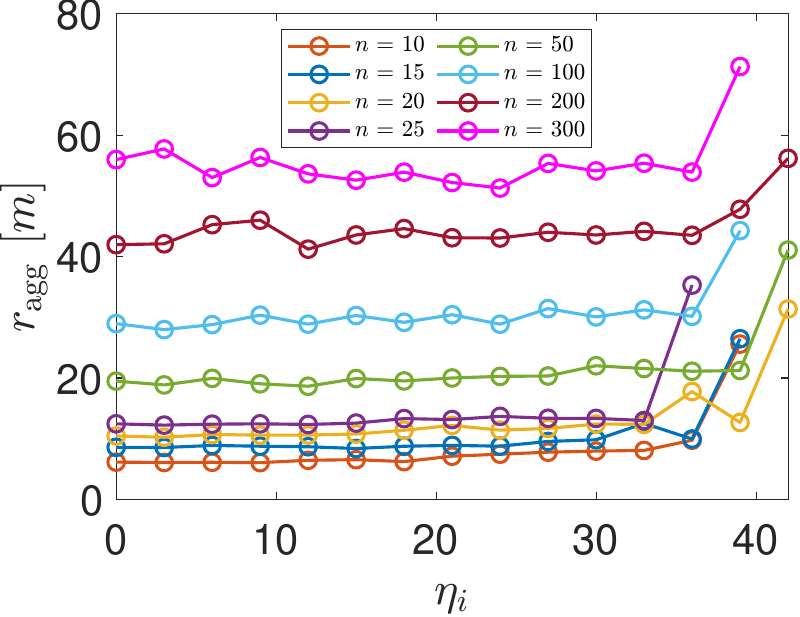}
        \subcaption[]{} \label{fig:analysis:disp}
    \end{subfigure}
    \begin{subfigure}{0.32\linewidth}
        \centering
        \includegraphics[width=\linewidth]{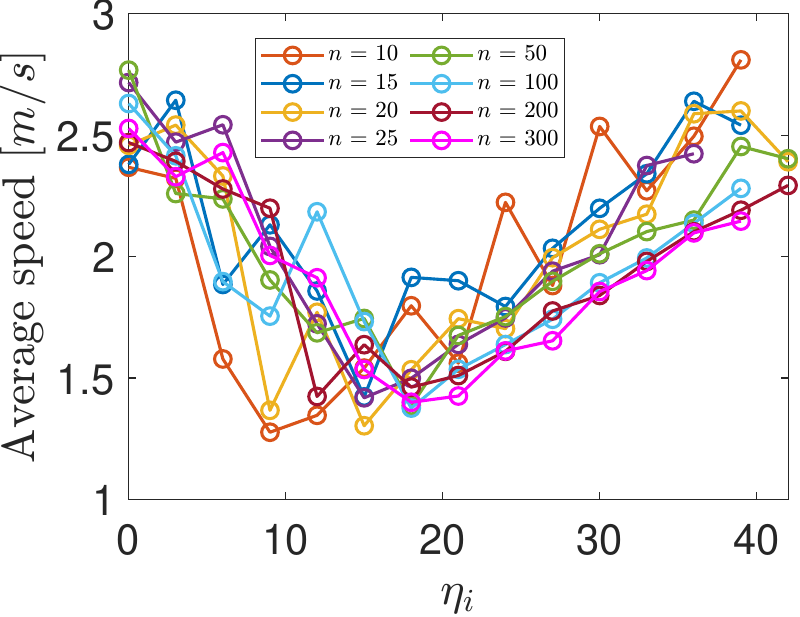}
        \subcaption[]{} \label{fig:analysis:speed}
    \end{subfigure}
    \caption{An overview of our model: (a) motion patterns and behavioral phase transition diagram, (b) radii of dispersion, and (c) average speeds. The highest \(\eta_i\) for each \(n\) represents the breakpoint where the aggregation of collective motion is no longer maintained. The initialization follows the same setup as in Fig.~\ref{fig:col_beh}, except for the initial positions for \(n=50, 100, 200, 300\), where the upper range is extended to \(20\, \mathrm{m}\), \(30\, \mathrm{m}\), \(50\, \mathrm{m}\), and \(75\, \mathrm{m}\), respectively. Given consistent outcomes across repeated trials, the results from a single representative run are hereby presented. 
}
    \label{fig:analysis}
\end{figure}

\begin{remark}
Vortexing behavior does not emerge explicitly through the kinetic offset parameter tuning across arbitrary population sizes. However, as illustrated in Fig.~\ref{fig:analysis}(\subref{fig:analysis:phase}), it consistently arises as a typical collective pattern in small-scale populations ($n<10$). Future work may explore model extensions or variations that allow vortexing behavior on broader population scales.
\end{remark}

\begin{remark}
To suppress the emergence of structured disorder—manifested as vortex-like motion, as illustrated in Fig.~\ref{fig:col_beh}(\subref{fig:col_beh:v1})—and instead promote disordered yet cohesive swarming behavior in low-density settings (e.g. $n < 10$), the parameter $\beta$ can be effectively increased. A larger $\beta$ amplifies the alignment weight $\phi(\cdot)$, thereby preventing the formation of structured motion patterns and facilitating the globally disordered dynamics characteristic of swarm-like behavior. For example, increasing $\beta$ from $1$ to $5$ transitions the collective behavior from that shown in Fig.~\ref{fig:col_beh}(\subref{fig:col_beh:v1}) to the swarm-like behavior depicted in Fig.~\ref{fig:col_beh}(\subref{fig:col_beh:s}).
\end{remark}

To illustrate the continuous effect of the kinetic offset $\eta_i$ on collective order across varying population sizes $n$, we present the phase diagram in Fig.~\ref{fig:phase}. It shows that both $\eta_i$ and $n$ induce phase transitions, but their impacts differ significantly. A phase transition occurs consistently as $\eta_i$ increases from 10 to 15 for all $n>10$. In contrast, for $n<10$, a phase transition is observed mostly when $\eta_i$ is between 0 and 10. Generally, lower values of $\eta_i$ yield higher alignment scores, promoting the faster and stronger group cohesion characteristic of flocking behavior. As $\eta_i$ increases, the alignment score decreases, indicative of weakening group cohesion, and the collective exhibits swarming-like characteristics. The phase diagram in Fig.~\ref{fig:phase} reveals a critical boundary of $\eta_i$ beyond which collective order disintegrates.

\begin{figure}[t]
    \centering

    \includegraphics[width=0.5\linewidth]{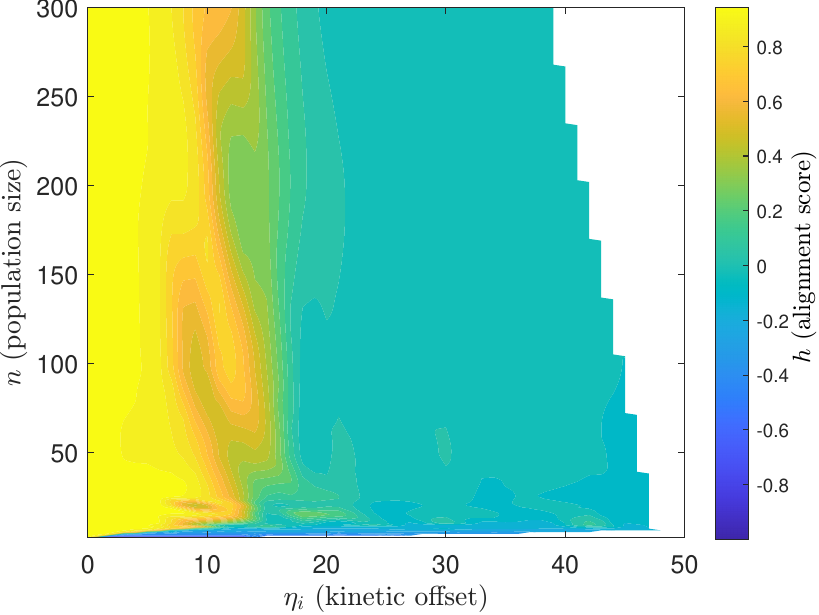}
 
    \caption{Phase transition diagram illustrating the alignment score \(h\) as a function of the kinetic offset $\eta_i$ and population size $n$. Contour levels represent the degree of alignment achieved by the collective for varying combinations of $\eta_i$ and $n$, revealing critical transition regions where collective behavior shifts from disordered to ordered states. The uncolored region corresponds to the critical $\eta_i$ values, beyond which aggregation in a population of size $n$ breaks down, as reflected in the dispersion radii shown in Fig.~\ref{fig:analysis}(\subref{fig:analysis:disp}). }
    \label{fig:phase}
\end{figure}

We further examine how spatial offset, representing the agents’ propensity to aggregate, influences the emergence and structure of collective motion. The average distance \(d_{\text{avg}}\) and minimum distance \(d_{\min}\) are calculated as follows,
\[
d_{\text{avg}} = \text{avg}(\sum_{\forall(i,j)} \|\mathbf{p}_j- \mathbf{p}_i\|),\quad d_{\min} = \min_{\forall(i,j)}\|\mathbf{p}_j- \mathbf{p}_i\|.
\]

In Fig.~\ref{fig:agg}, we plot the minimum distance, average distance, and radii of aggregation for the collective motion of \(n = 100\) agents evolving over \(t = 30\) s, with gradually increasing values of \(\delta_i\). Figs.~\ref{fig:agg}(\subref{fig:agg:dmin3})–\ref{fig:agg}(\subref{fig:agg:r3}) correspond to \(\eta_i = 3\), which produces a relatively ordered collective motion, exhibiting flocking behavior. As seen in Fig.~\ref{fig:agg}(\subref{fig:agg:dmin3}), a nonzero \(\delta_i\) prevents inter-agent collisions and results in a relatively stable separation distance. Both the average distances and radii of aggregation increase as \(\delta_i\) increases, with \(\delta_i = 0.5\) and \(1\) leading to the smoothest flocking behavior. Figs.~\ref{fig:agg}(\subref{fig:agg:dmin15})–\ref{fig:agg}(\subref{fig:agg:r15}) correspond to \(\eta_i = 21\), which leads to disordered collective motion, exhibiting swarming behavior. Although the evolution of disorder engenders random motion directions at the individual level and consequently results in difficulties in establishing a stable minimum distance, collision avoidance is nevertheless preserved, as observed from Fig.~\ref{fig:agg}(\subref{fig:agg:dmin15}). Notably, \(\delta_i = 0.5\) and \(1\) promote safer and more stable swarming behavior, as the fluctuations in minimum distances are minimized. Accordingly, Fig.~\ref{fig:agg} illustrates that the degree of aggregation in collective motion is directly modulated by the agents’ propensity to aggregate, as parameterized by the spatial offset $\delta_i$. Moreover, a nonzero \(\delta_i\) effectively prevents collisions, with the desired results observed for \(\delta_i = 0.5\) and \(1\) in both flocking and swarming behaviors.

\begin{figure}[htbp]
    \centering
    \begin{subfigure}{0.32\linewidth}
        \centering
        \includegraphics[width=\linewidth]{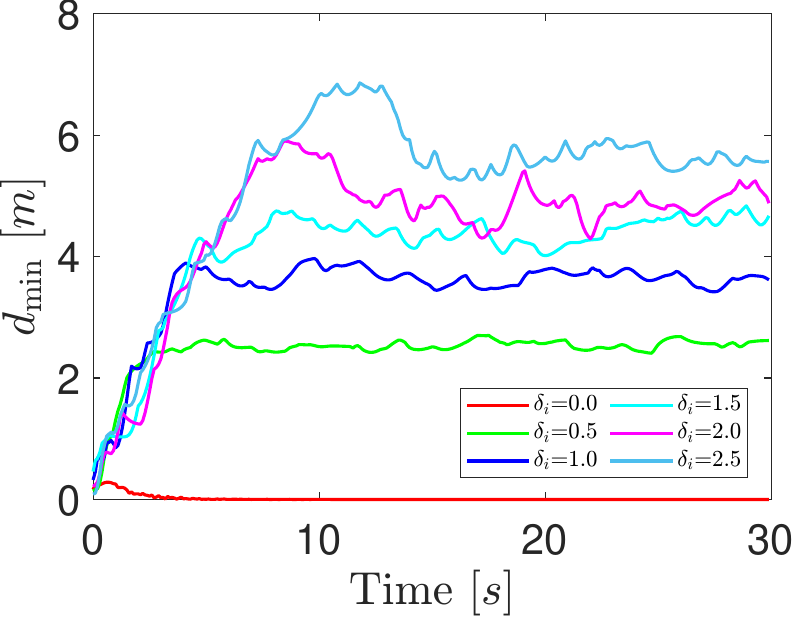}
        \subcaption[]{} \label{fig:agg:dmin3}
    \end{subfigure}
    \begin{subfigure}{0.32\linewidth}
        \centering
        \includegraphics[width=\linewidth]{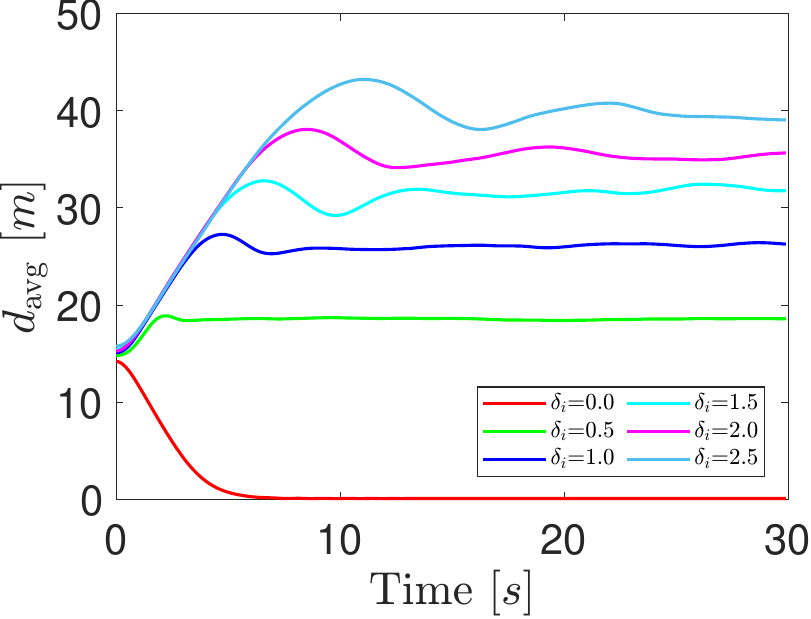}
        \subcaption[]{} \label{fig:agg:davg3}
    \end{subfigure}
    \begin{subfigure}{0.32\linewidth}
        \centering
        \includegraphics[width=\linewidth]{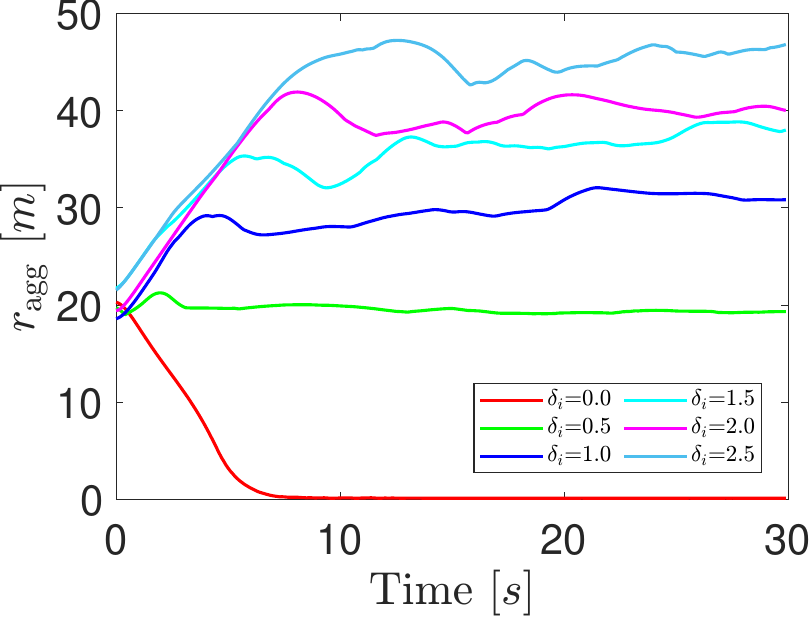}
        \subcaption[]{} \label{fig:agg:r3}
    \end{subfigure}
    \begin{subfigure}{0.32\linewidth}
        \centering
        \includegraphics[width=\linewidth]{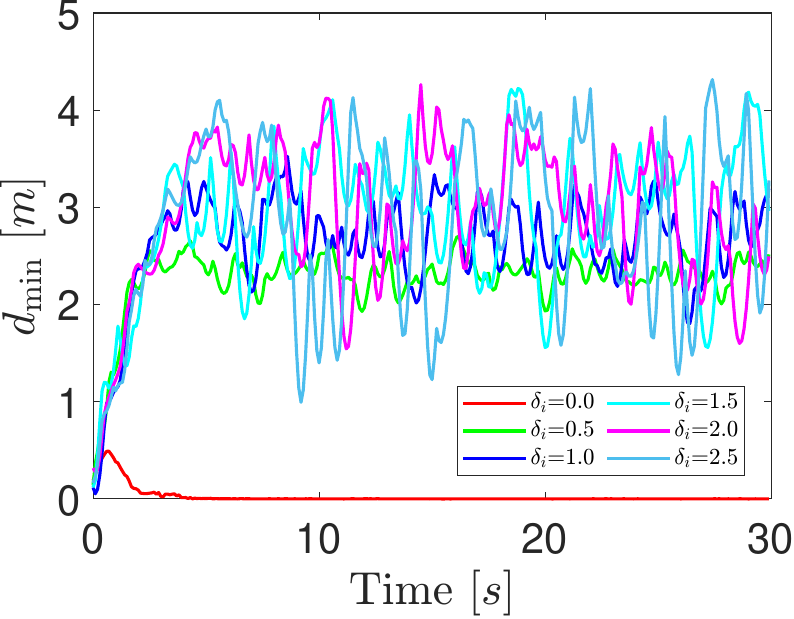}
        \subcaption[]{} \label{fig:agg:dmin15}
    \end{subfigure}
    \begin{subfigure}{0.32\linewidth}
        \centering
        \includegraphics[width=\linewidth]{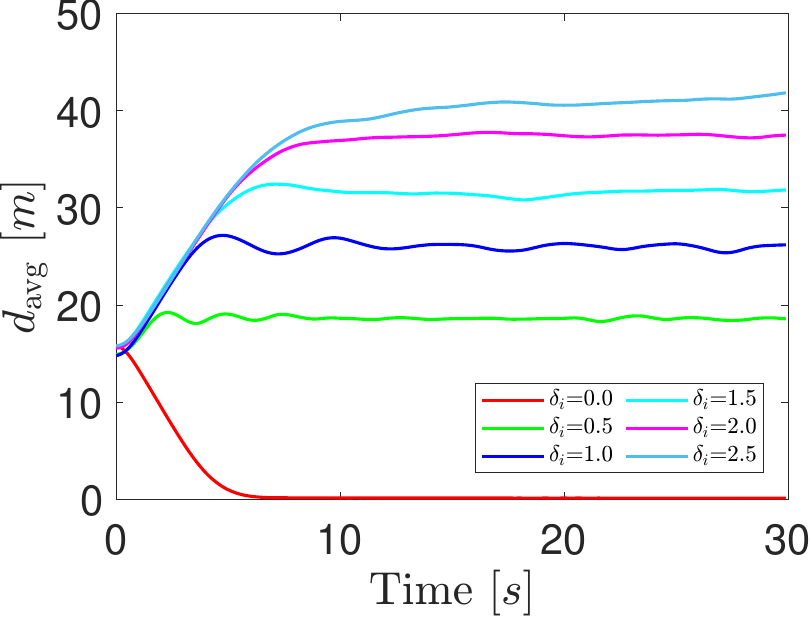}
        \subcaption[]{} \label{fig:agg:davg15}
    \end{subfigure}
    \begin{subfigure}{0.32\linewidth}
        \centering
        \includegraphics[width=\linewidth]{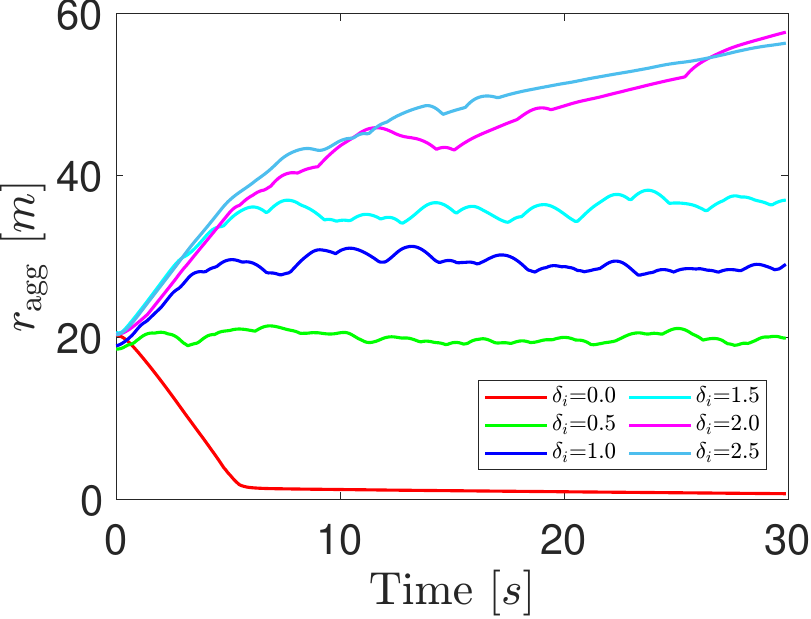}
        \subcaption[]{} \label{fig:agg:r15}
    \end{subfigure}
    \caption{Overview of collective motion for \(n = 100\) agents over \(t = 30\, \mathrm{s}\) with a time step of \(0.1\, \mathrm{s}\) and varying values of \(\delta_i\). (a) and (d) show the minimum distance, (b) and (e) show the average distance, and (c) and (f) show the radii of aggregation. The results (a) to (c) correspond to \(\eta_i = 3\), where the agents exhibit flocking behavior, and (d) to (f) correspond to \(\eta_i = 21\), where agents display swarming behavior.
}
    \label{fig:agg}
\end{figure}

From Fig.~\ref{fig:agg}, the evolution of the spatial dispersion of the group, as indicated by the minimum distance, the average distance, and the radii of aggregation, demonstrates the impact of varying spatial offset $\delta_i$ on collective movement. Specifically, lower values of $\delta_i$ result in more compact formations, which is favorable for flocking behavior. In contrast, higher values of $\delta_i$ increase spatial separation between agents, fostering swarming behavior marked by greater individual dispersion and reduced coordination. These dynamics emphasize the role of $\delta_i$ in modulating the transition between cohesive flocking and dispersed swarming states. However, it is important to note that the primary factor driving group behavioral transitions is the kinetic offset $\eta_i$, as discussed earlier.

\begin{remark}
In biological swarms and flocks, collisions typically occur when reaction times are insufficient, initial velocities are excessively high, or inter-individual distances are too small. The collision avoidance mechanism in the proposed model exhibits a more naturalistic character compared to the explicitly imposed repulsive interactions in the Olfati-Saber model. Provided that initial velocities remain within a moderate range and agents are not excessively proximate, the model generally maintains effective collision avoidance. This behavior can be further modulated via the spatial offset parameter \(\delta_i\), which regulates the repulsive tendency arising from spatial proximity.
\end{remark}

The proposed model ensures that agents consistently establish and maintain a personal space (when \(\delta_i > 0\)), effectively preventing collisions or excessive proximity to their neighbors. This self-organizing behavior emerges naturally from the model dynamics and does not depend on carefully selected initial conditions. Even in simulations where agents are initialized with overlapping or colliding positions, the model rapidly enforces separation, allowing agents to establish their personal space almost instantaneously. This robustness highlights the model's ability to manage collision-free motion regardless of the initial configuration of the agents.

\subsection{Cluttered environment}
 A simulation of 10 agents in a $100 \times 100\,\text{m}$ environment with obstacles was conducted to validate the target-directed collective behavior model in a cluttered environment. The agents aimed to reach a target at $(90, 90)\,\text{m}$ while avoiding three circular obstacles located at $(25, 30)\,\text{m}$, $(50, 40)\,\text{m}$, and $(90, 80)\,\text{m}$, each with a radius of \(5\,\text{m}\) , as shown in Fig.~\ref{fig:obs-avo}. The parameters were set as follows: $r_i = 10\,\text{m}$, $\delta_i = 1$, $\eta_i = 0.5$, $\alpha = 2$, $\beta = 1$, $\kappa_i = 0.5$, $c_i = 15\,\text{m}$, and $\sigma = 3$. Over a \(40\,\text{s}\) simulation, the agents demonstrated cohesive, collision-free motion, effectively balancing aggregation, alignment, and obstacle avoidance. The flock temporarily fragmented near the second obstacle but later regrouped and successfully reached the target. The choice of \(\eta_i = 0.5\) accounts for environmental noise, ensuring stable alignment under realistic conditions.
 
\begin{figure*}[htbp]
    \centering
    \begin{subfigure}{0.22\linewidth}
        \centering
        \includegraphics[width=\linewidth]{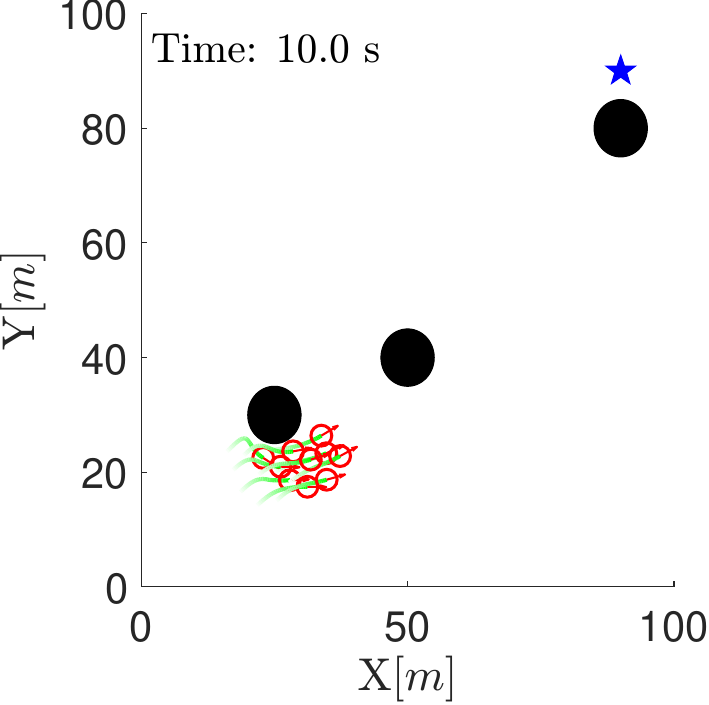}
        \subcaption[]{} 
    \end{subfigure}
    \begin{subfigure}{0.22\linewidth}
        \centering
        \includegraphics[width=\linewidth]{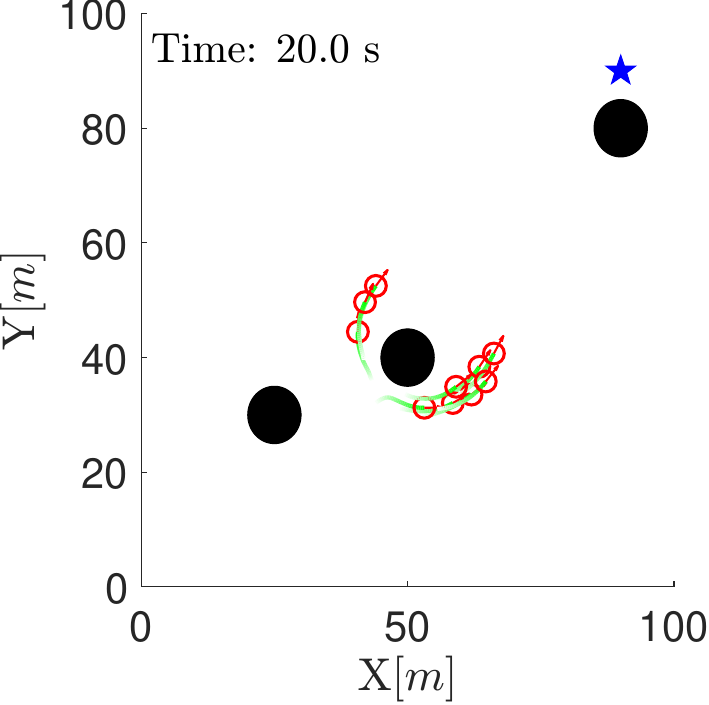}
        \subcaption[]{} 
    \end{subfigure}
    \begin{subfigure}{0.22\linewidth}
        \centering
        \includegraphics[width=\linewidth]{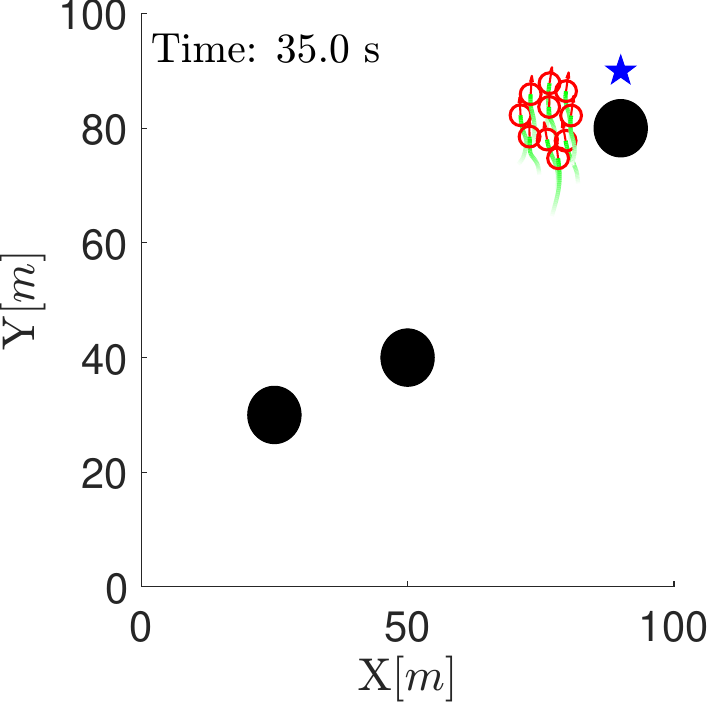}
        \subcaption[]{} 
    \end{subfigure}
    \begin{subfigure}{0.22\linewidth}
        \centering
        \includegraphics[width=\linewidth]{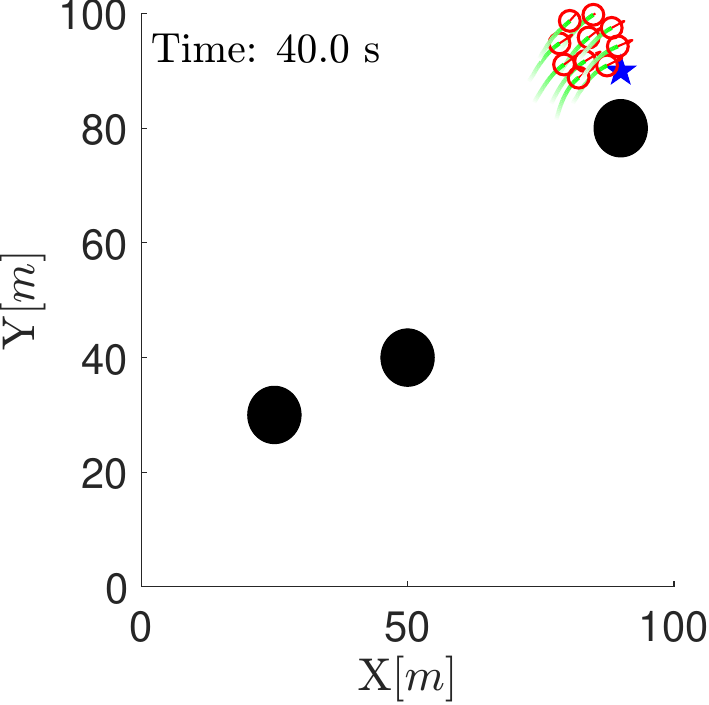}
        \subcaption[]{} 
    \end{subfigure}
    \caption{Snapshots of 10 agents navigating a $100\times100\,\text{m}$ environment toward a target at \((90,90)\,\text{m}\), avoiding obstacles at \((25,30)\,\text{m}\), \((50,40)\,\text{m}\), and \((90,80)\,\text{m}\), each with a \(5\,\text{m}\) radius. With initial positions in \([0,10]\,\text{m}\), subfigures depict: (a) encounter with the obstacle at \((25,30)\,\text{m}\), (b) encounter with the obstacle at \((50,40)\,\text{m}\), (c) encounter with the obstacle at \((90,80)\,\text{m}\), and (d) arrival at the target. Trajectories are shown in green, agent positions as red circles, velocities as red arrows, the target as a blue star, and obstacles as black circles.}

    \label{fig:obs-avo}
\end{figure*}

\subsection{Energy-driven adaptation}
To evaluate the cognitive model with the energy-driven adaptation mechanism, we define the parameters \(\delta_{\text{min}}\), \(\delta_{\text{max}}\), \(\eta_{\text{min}}\), and \(\eta_{\text{max}}\), allowing agents to adjust their policies \(\delta_i\) and \(\eta_i\) based on their energy level \(E_i\). From Fig.~\ref{fig:analysis}(\subref{fig:analysis:disp}), we define \(\eta_{\max} = 13\) for \(n = 2\), \(\eta_{\max} = 26\) for \(n = 3\), and \(\eta_{\max} = 33\) for \(n > 3\). However, to ensure robustness in the presence of varying \(\eta_i\) among neighbors, a more conservative choice is \(\eta_{\max} = 15\), as extreme values may not sustain aggregation. Additionally, Fig.~\ref{fig:agg} provides an intuitive basis for determining \(\delta_{\max}\). Specifically, \(\delta_{\max}=2.5\) for flocking and \(\delta_{\max}=1.5\) for swarming. Based on these observations, we set \(\delta_{\max}=2\) and \(\delta_{\min}=0.5\) as reasonable values, ensuring that agents remain within their interaction zone radius \(r_i\). 

We run the model with \(n = 20\) agents, both with and without the adaptation mechanism enabled. Fig.~\ref{fig:eng}(\subref{fig:eng:eta:f})–(\subref{fig:eng:E:f}) presents the time histories of \(\eta_i\), \(\delta_i\), and \(E_i\) for flocking behavior in a compact formation, where agents exhibit strong alignment and maintain close proximity to their neighbors. Fig.~\ref{fig:eng}(\subref{fig:eng:eta:s})-(\subref{fig:eng:E:s}) displays the corresponding time histories for swarming behavior, where agents are driven to explore the environment. Finally, Fig.~\ref{fig:eng}(\subref{fig:eng:eta:ada})-(\subref{fig:eng:E:ada}) illustrates the time histories of \(\eta_i\), \(\delta_i\), and \(E_i\) when agents dynamically adapt \(\eta_i\) and \(\delta_i\) based on their energy levels \(E_i\). The snapshots of collective motion under the energy-driven adaptive mechanism are depicted in Fig.~\ref{fig:cog_mot}, illustrating how agents individually adjust their \(\delta_i\) and \(\eta_i\) policies in response to their own and their neighbors' energy levels, thereby shaping the dynamics of collective behavior.

From Fig.~\ref{fig:eng}(\subref{fig:eng:E:f}), Fig.~\ref{fig:eng}(\subref{fig:eng:E:s}), and Fig.~\ref{fig:eng}(\subref{fig:eng:E:ada}), three key observations emerge: (1) flocking in a compact formation conserves energy, (2) swarming behavior incurs higher energy expenditure, and (3) the energy-driven adaptation mechanism effectively regulates energy efficiency by transitioning agents from energy-intensive swarming to energy-conserving flocking when their energy drops below a threshold. This highlights the mechanism’s ability to optimize agent behavior under energy constraints. 

\begin{figure}[htbp]
    \centering
    \begin{subfigure}{0.24\linewidth}
        \centering
        \includegraphics[width=\linewidth]{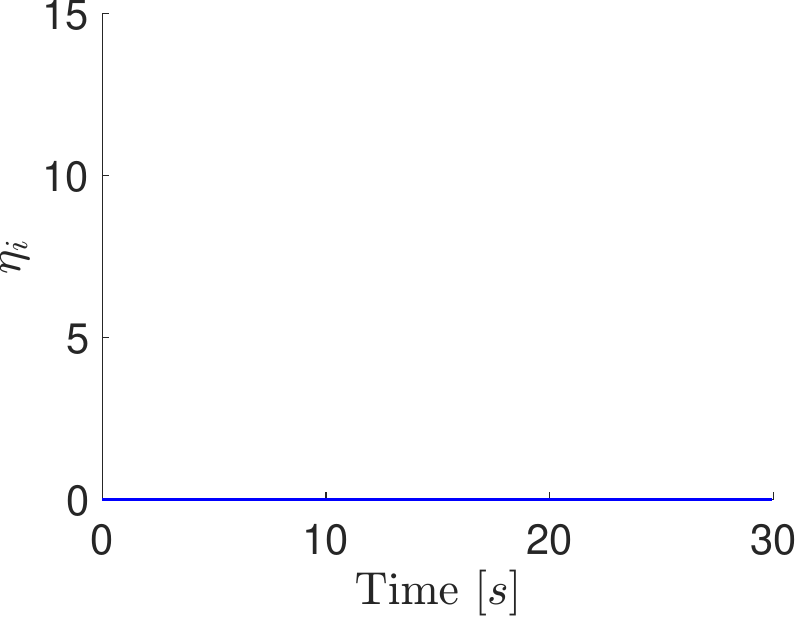}
        \subcaption[]{} \label{fig:eng:eta:f}
    \end{subfigure}
    \begin{subfigure}{0.24\linewidth}
        \centering
        \includegraphics[width=\linewidth]{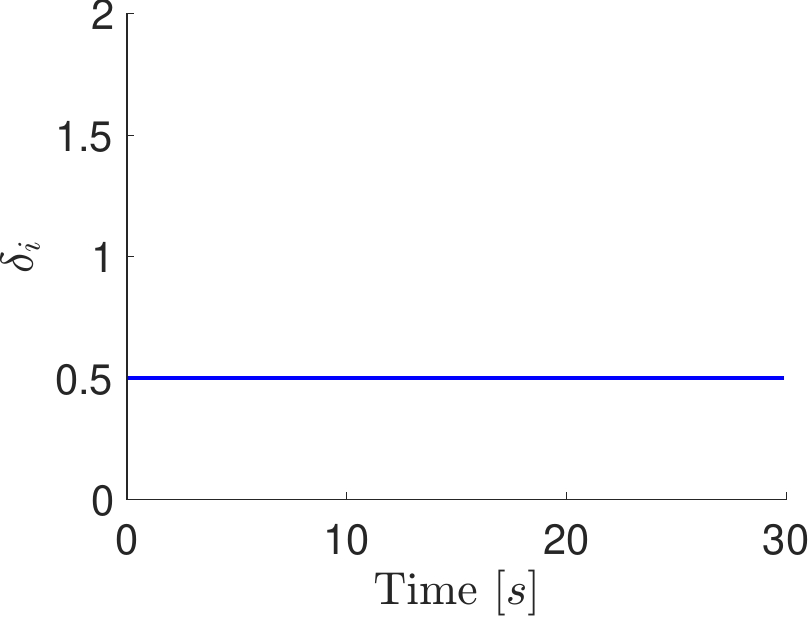}
        \subcaption[]{} \label{fig:eng:del:f}
    \end{subfigure}
    \begin{subfigure}{0.24\linewidth}
        \centering
        \includegraphics[width=\linewidth]{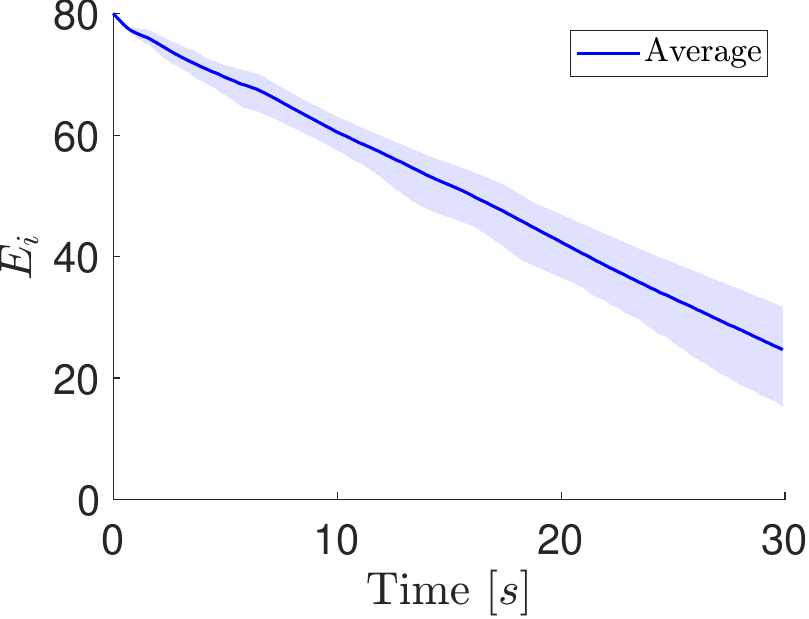}
        \subcaption[]{} \label{fig:eng:E:f}
    \end{subfigure}
    \begin{subfigure}{0.24\linewidth}
        \centering
        \includegraphics[width=\linewidth]{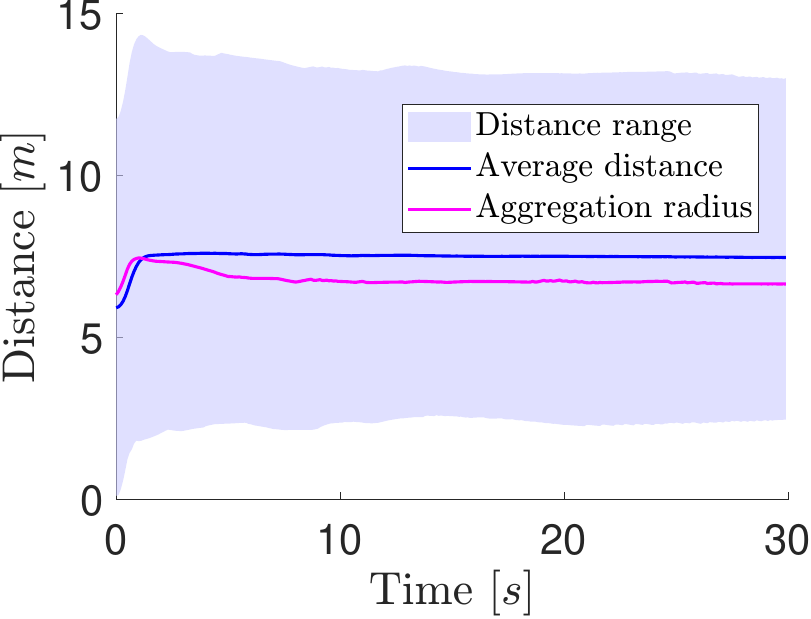}
        \subcaption[]{} \label{fig:eng:d:f}
    \end{subfigure}
    \begin{subfigure}{0.24\linewidth}
        \centering
        \includegraphics[width=\linewidth]{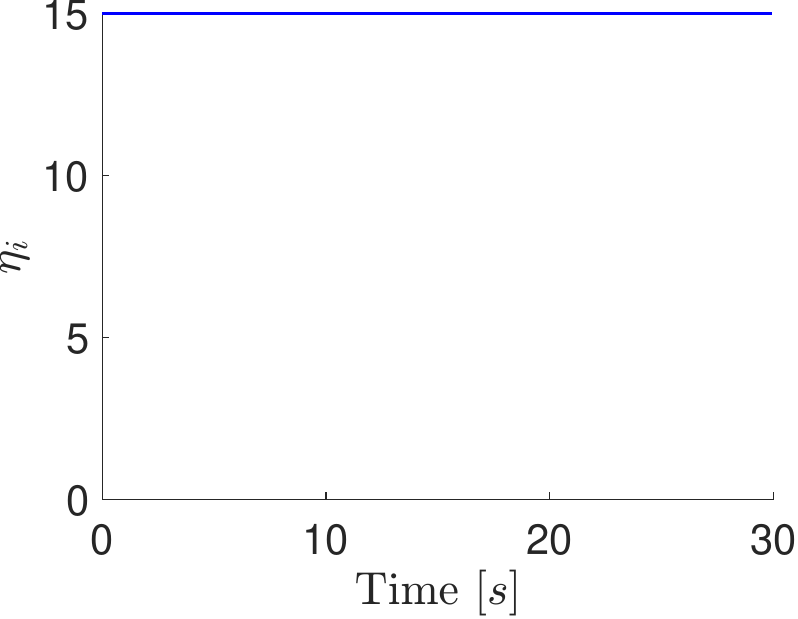}
        \subcaption[]{} \label{fig:eng:eta:s}
    \end{subfigure}
    \begin{subfigure}{0.24\linewidth}
        \centering
        \includegraphics[width=\linewidth]{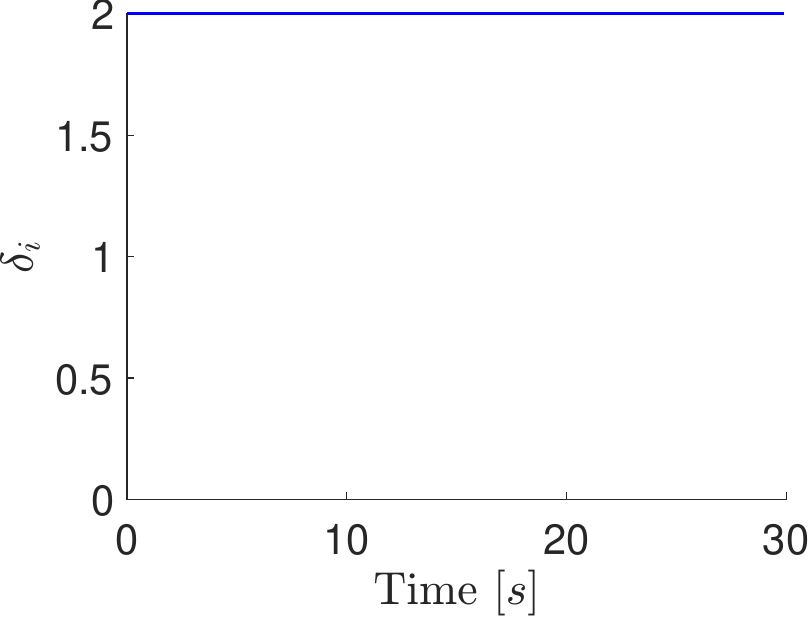}
        \subcaption[]{} \label{fig:eng:del:s}
    \end{subfigure}
    \begin{subfigure}{0.24\linewidth}
        \centering
        \includegraphics[width=\linewidth]{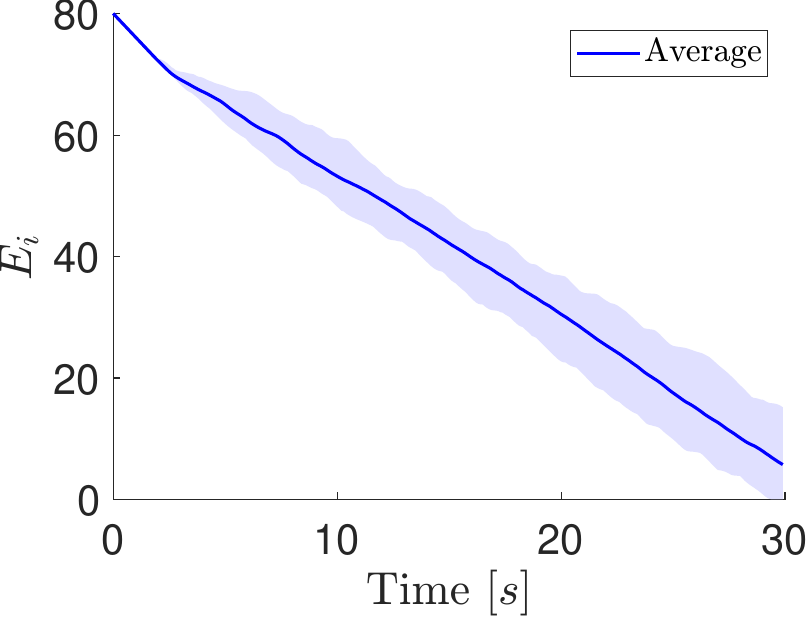}
        \subcaption[]{} \label{fig:eng:E:s}
    \end{subfigure}
    \begin{subfigure}{0.24\linewidth}
        \centering
        \includegraphics[width=\linewidth]{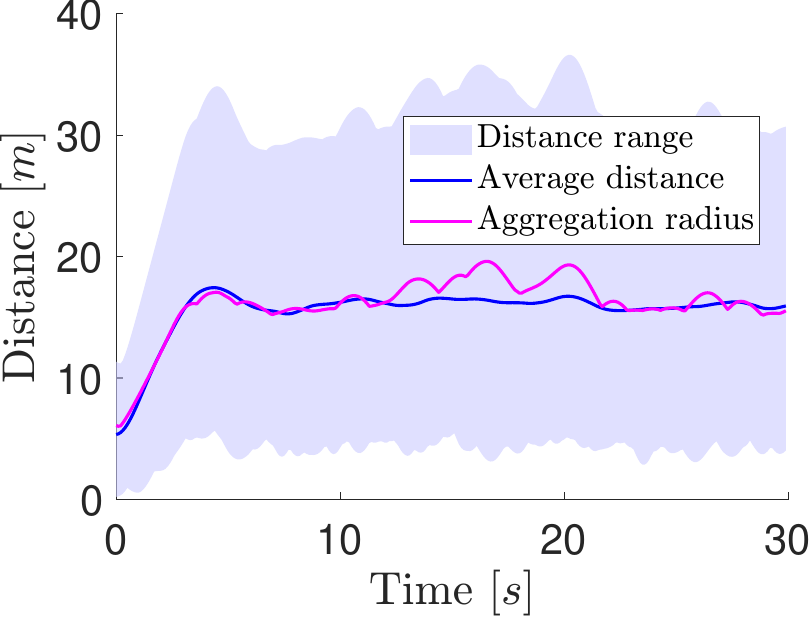}
        \subcaption[]{} \label{fig:eng:d:s}
    \end{subfigure}
    \begin{subfigure}{0.24\linewidth}
        \centering
        \includegraphics[width=\linewidth]{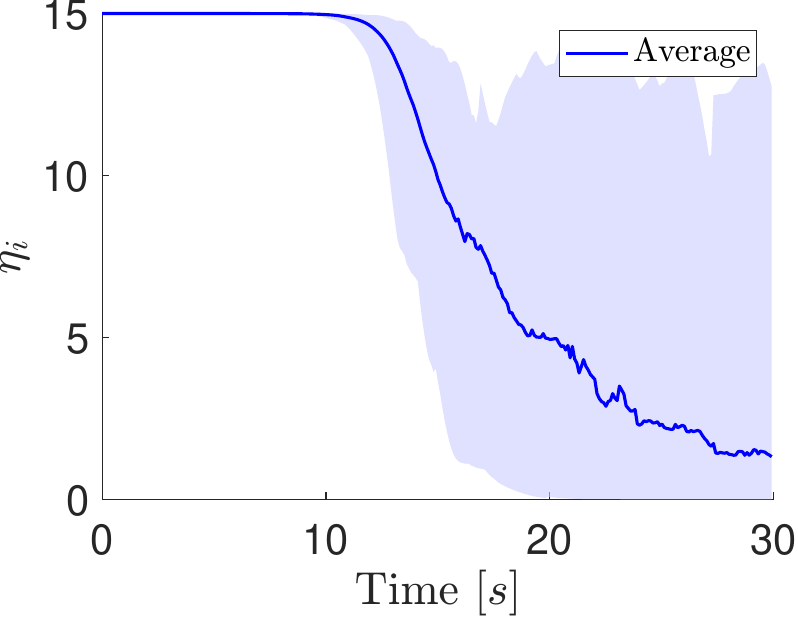}
        \subcaption[]{} \label{fig:eng:eta:ada}
    \end{subfigure}
    \begin{subfigure}{0.24\linewidth}
        \centering
        \includegraphics[width=\linewidth]{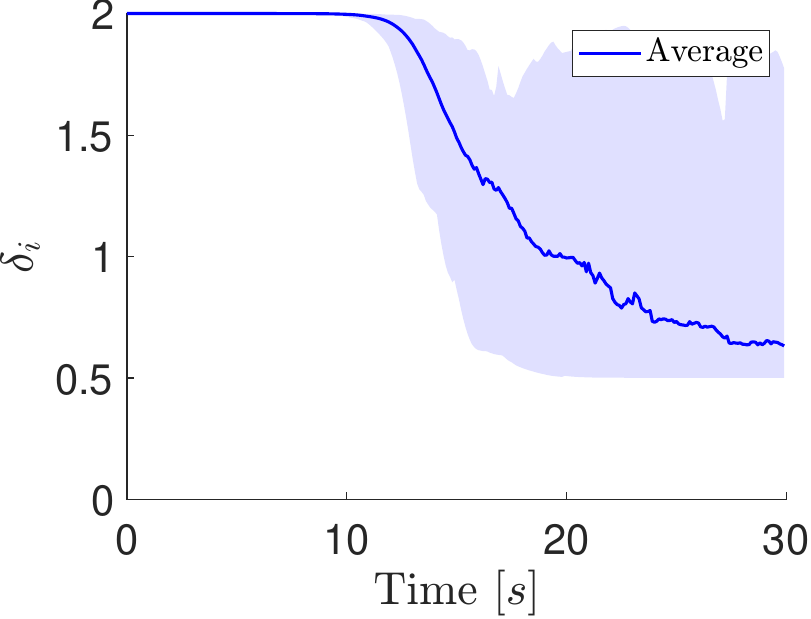}
        \subcaption[]{} \label{fig:eng:del:ada}
    \end{subfigure}
    \begin{subfigure}{0.24\linewidth}
        \centering
        \includegraphics[width=\linewidth]{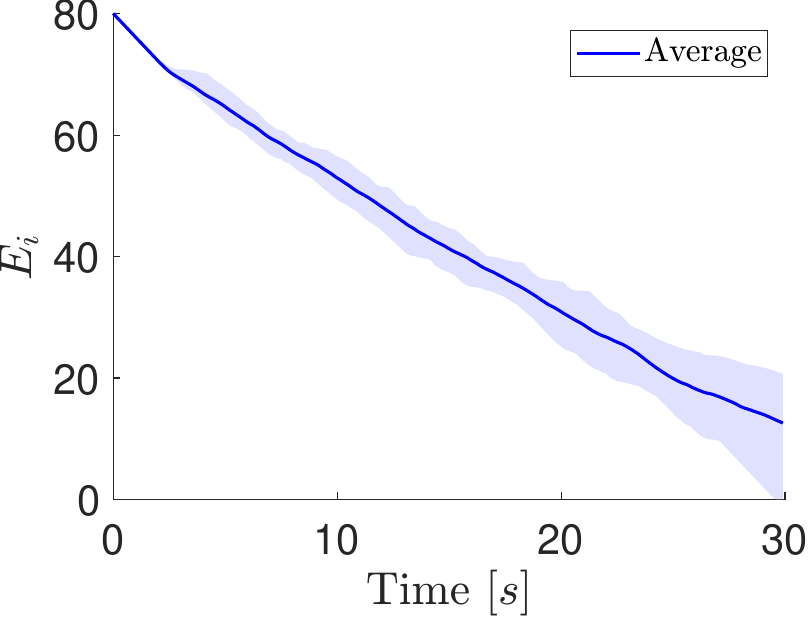}
        \subcaption[]{} \label{fig:eng:E:ada}
    \end{subfigure}
    \begin{subfigure}{0.24\linewidth}
        \centering
        \includegraphics[width=\linewidth]{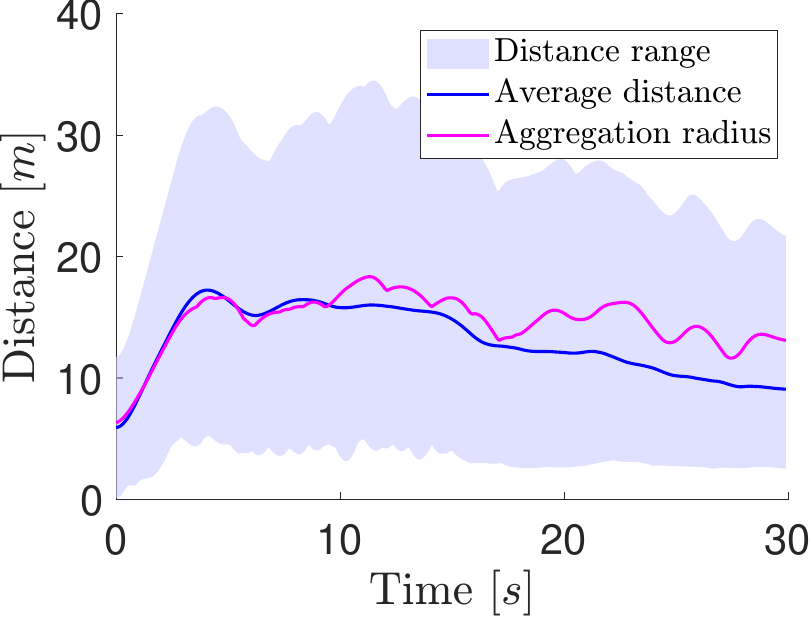}
        \subcaption[]{} \label{fig:eng:d:ada}
    \end{subfigure}
    \caption{Time histories of \(\eta_i\), \(\delta_i\), \(E_i\), relative distances and aggregation radius for \(n=20, E_{i0}=80, c_{i1}=0.15,c_{i2}=0.015,k_{\delta}=k_{\eta}=0.5\). (a)-(c) show the evolution of \(\eta_i\), \(\delta_i\), and \(E_i\) for flocking behavior with a compact formation, where agents maintain strong alignment and close proximity. (e)-(g) depict the corresponding time histories for swarming behavior, where agents explore the environment. (i)-(k) illustrate the adaptive case, where \(\eta_i\) and \(\delta_i\) are dynamically adjusted based on energy levels \(E_i\). The threshold energy level is set to \(E_{\text{th}}=40\). The results demonstrate that flocking conserves energy, swarming expends more energy, and the energy-driven adaptation mechanism efficiently transitions agents from swarming to flocking to optimize energy usage. Additionally, from (d), (h), and (l), collision avoidance and group aggregation are evident.}
    \label{fig:eng}
\end{figure}

\begin{figure*}[htbp]
    \centering
    \begin{subfigure}{0.22\linewidth}
        \centering
        \includegraphics[width=\linewidth]{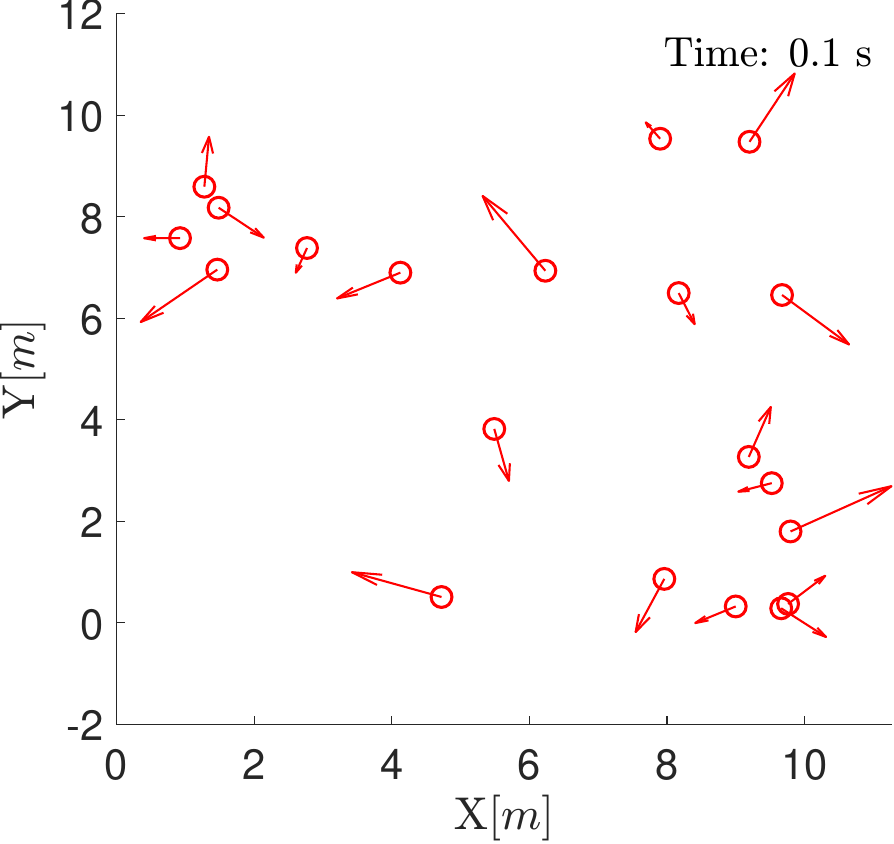}
        \subcaption[]{} \label{fig:cog_mot:001}
    \end{subfigure}
    \begin{subfigure}{0.22\linewidth}
        \centering
        \includegraphics[width=\linewidth]{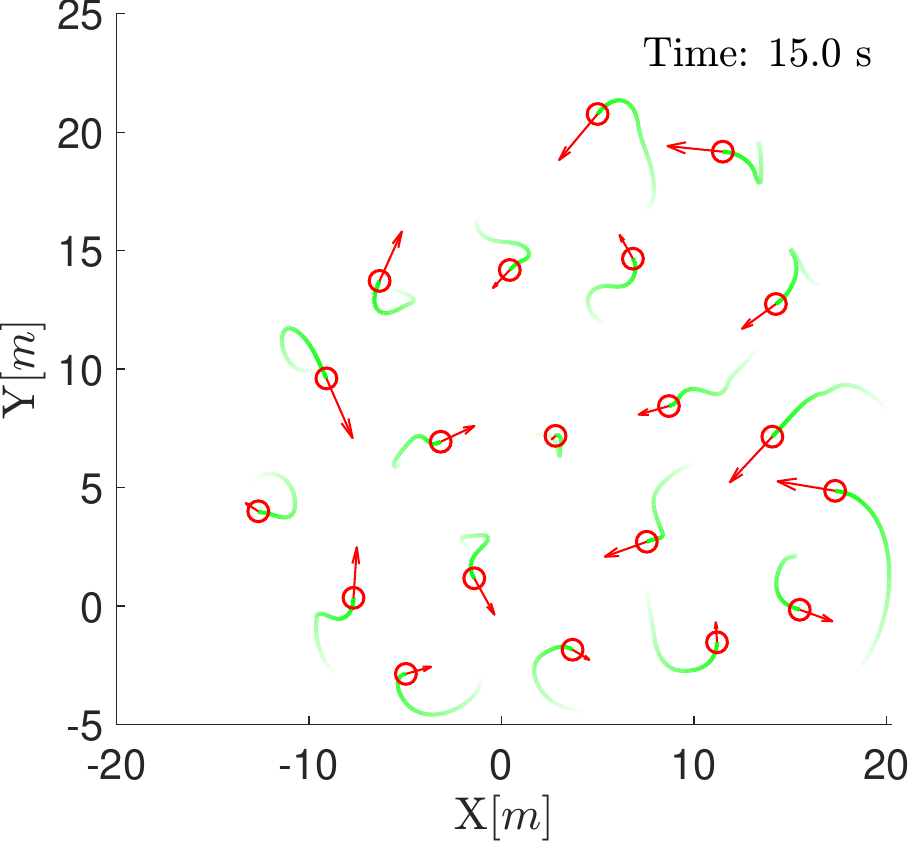}
        \subcaption[]{} \label{fig:cog_mot:140}
    \end{subfigure}
    \begin{subfigure}{0.22\linewidth}
        \centering
        \includegraphics[width=\linewidth]{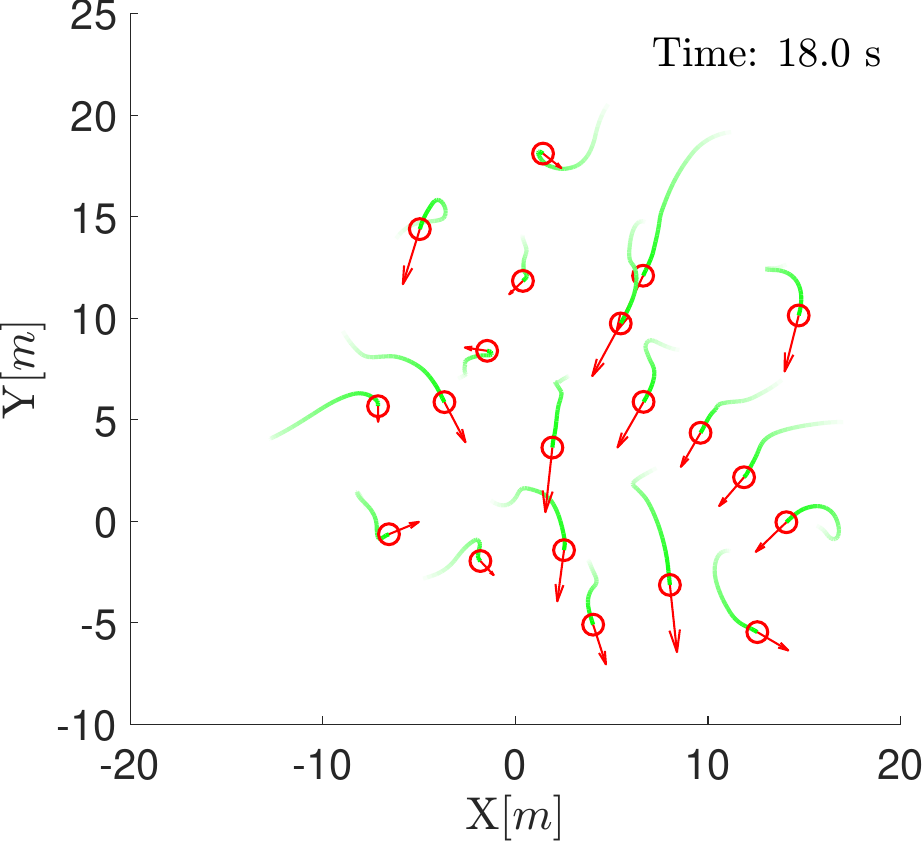}
        \subcaption[]{} \label{fig:cog_mot:250}
    \end{subfigure}
    \begin{subfigure}{0.22\linewidth}
        \centering
        \includegraphics[width=\linewidth]{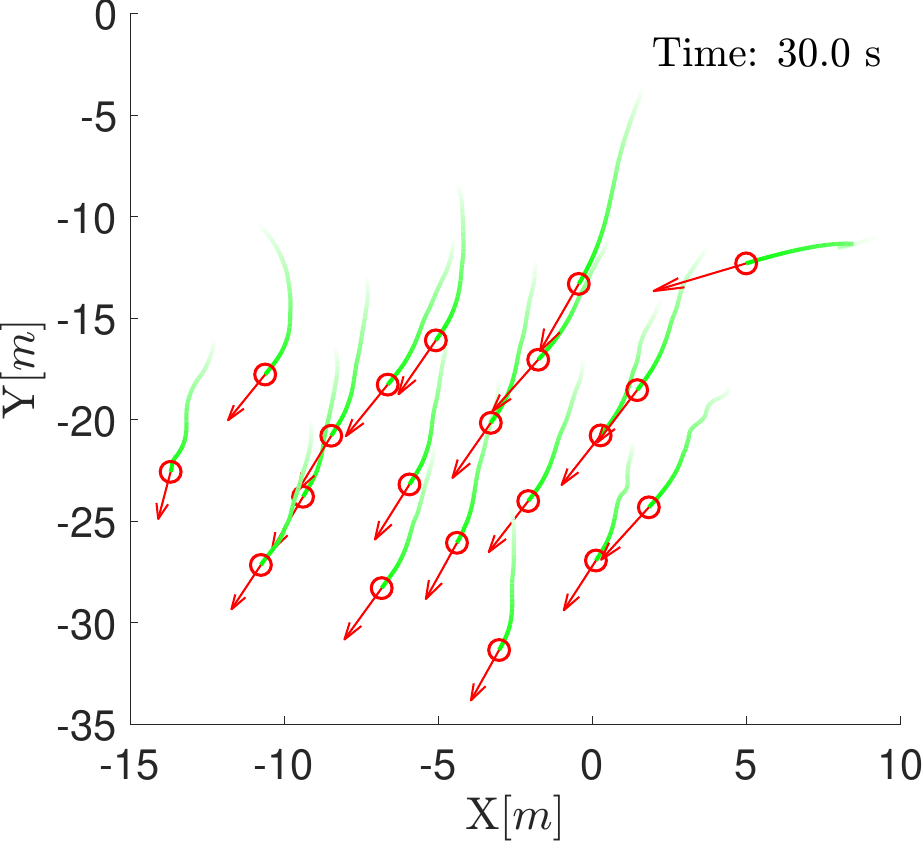}
        \subcaption[]{} \label{fig:cog_mot:300}
    \end{subfigure}
    \caption{Snapshots of collective motion under the energy-driven adaptive mechanism. Agents dynamically adjust their \(\delta_i\) and \(\eta_i\) policies in response to their own and neighbors' energy levels, shaping the emergent collective behavior. Initialized with randomized positions and velocities in (a), the collective behavior transitions to swarming in (b) and evolves into flocking in (c) to (d).}
    \label{fig:cog_mot}
\end{figure*}

By utilizing the adaptive mechanism, agents autonomously determine their actions based on their own energy states and those of their neighbors. An alternative adaptation strategy could involve agents responding to their own energy levels as well as the actions of their neighbors (i.e., \(\delta_j, \eta_j\) for every \(j \in \mathcal{N}_i\)). 

This adaptive mechanism allows the group to exhibit emergent intelligence, such as transitioning between swarming and flocking behaviors, resembling collective cognition observed in animal groups. Moreover, this approach provides a robust decentralized framework for designing robotic multi-agent systems capable of demonstrating adaptive, collective intelligence.

The scalability of the proposed framework is evidenced by its consistent performance across varying population sizes. As demonstrated by simulations, the model reliably reproduces collective behaviors using only local interactions, and these emergent behaviors remain qualitatively robust even in large swarms. Consequently, the model exhibits the desirable scalability properties required for large-scale multi-agent systems.

\section{Conclusion and Future Work}\label{sec:con}

This paper presented a minimal yet expressive mathematical model capable of replicating complex collective behaviors observed in animal groups, such as swarming, vortexing, and flocking. Using simple interaction rules based on inter-agent distance, speed, direction, and local density, the model captures the fundamental mechanisms of alignment, separation, and aggregation. A tunable kinetic offset parameter enables smooth transitions between behaviors, and an energy-aware adaptive mechanism governs cognitive switching. Despite its simplicity, the model produces rich, collision-free dynamics suitable for decentralized, energy-adaptive swarm robotics.

While the framework is designed for lightweight, decentralized computation, real-world deployment poses practical challenges. Sensor noise and communication delays can introduce inaccuracies or outdated information. Physical and environmental constraints, such as bounded actuation and obstacles, further necessitate control saturation schemes to ensure effective low-level control and real-time obstacle avoidance. In addition, future work will focus on extending the model within the RBC framework supported by the E-CARGO ontology, enabling a hybrid architecture that combines continuous behavior modulation with formalized role coordination. Such integration is expected to enhance robustness and scalability in autonomous multi-agent systems.

\section*{Acknowledgments} This work was funded by the Czech Science Foundation (GAČR) under research project no. \(\mathrm{23-07517S}\) and the European Union under the project Robotics and Advanced Industrial Production (reg. no. \(\mathrm{CZ.02.01.01/00/22\_008/0004590}\)).

\end{document}